\documentclass{article}

\PassOptionsToPackage{numbers, compress}{natbib}

\usepackage[preprint]{neurips_2026}

\usepackage[utf8]{inputenc}
\usepackage[T1]{fontenc}
\usepackage{hyperref}
\usepackage{url}
\usepackage{booktabs}
\usepackage{graphicx}
\usepackage{wrapfig}
\usepackage{amsmath}
\usepackage{amssymb}
\usepackage{amsfonts}
\usepackage{amsthm}
\usepackage{dsfont}
\usepackage{nicefrac}
\usepackage{microtype}
\usepackage[table]{xcolor}
\usepackage{subcaption}
\usepackage{tabularx}
\usepackage[normalem]{ulem}
\usepackage{algorithm}
\usepackage{algpseudocode}
\theoremstyle{plain}
\newtheorem{theorem}{Theorem}
\newtheorem{lemma}{Lemma}

\newtheorem{corollary}{Corollary}
\theoremstyle{definition}
\newtheorem{definition}{Definition}
\newtheorem{construction}{Construction}
\theoremstyle{remark}
\newtheorem{remark}{Remark}

\title{Hypergraph Pattern Machine: Compositional Tokenization for Higher-Order Interactions}

\author{  
\textbf{Kyrie Zhao}$^{1,*,\ddag}$\;
\textbf{Zehong Wang}$^{1,*,\dagger}$\;
\textbf{Tianyi Ma}$^1$\;
\textbf{Fang Wu}$^2$\; 
\textbf{Xiangru Tang}$^3$\\ 
\textbf{Pietro Li\`o}$^4$\;
\textbf{Sheng Wang}$^5$\;
\textbf{Yanfang Ye}$^{1,\dagger}$\\
$^1$University of Notre Dame\;
$^2$Stanford University\;
$^3$Yale University\; \\
$^4$University of Cambridge
$^5$University of Washington\\
\textsuperscript{$*$}{Equal Contribution}\;
\textsuperscript{$\dagger$}{Corresponding Authors}\;
\textsuperscript{$\ddag$}{Visiting Student}\;
\\
\texttt{<zwang43,yye7>@nd.edu} \\
}
\begin{document}

\maketitle

\begin{abstract}
Hypergraphs model higher-order relations that drive real-world decisions, from drug prescriptions to recommendations.
A central structural signal in such data, beyond what pairwise relations can express, is \textit{interaction compositionality}: whether a higher-order relation is compositional, emergent, or inhibitory with respect to its observed or unobserved sets.
In polypharmacy, the regime decides whether a drug should be dropped, kept, or excluded: a compositional drug triple can be safely simplified, an emergent triple requires all drugs jointly, and an inhibitory triple flags a drug that disrupts an existing interaction.
However, existing hypergraph learning methods, which merely propagate messages over observed hyperedges, leave this compositional signal unmodeled, allowing dangerous drug combinations to slip through and be misclassified.
To this end, we propose the Hypergraph Pattern Machine (HGPM), shifting the paradigm from message passing to learning the compositional pattern of subsets.
It tokenizes compositional subsets, organizes them in an inclusion DAG, and trains an inclusion-aware Transformer under masked reconstruction.
On ten hypergraph benchmarks, HGPM matches or exceeds state-of-the-art methods. 
Notably, in a real adverse-event prediction case, HGPM correctly identifies the drug addition that inhibits the side effect among feature-identical candidates, a discrimination existing methods cannot make.
The code and data are in \url{https://github.com/KryieZhao/HGPM.git}. 
\end{abstract}

\section{Introduction}

Hypergraphs~\citep{feng2019hgnn,chien2022allset,wang2023edhnn,tang2024hypergraphmlp,gao2025hyperfm,ma2025adaptive} model higher-order relations, e.g., drug interactions, recommendation, knowledge bases. Hypergraph data carries a structural property that pairwise graphs~\citep{kipf2017gcn,velickovic2018gat} cannot express~\citep{hajij2023topological,bodnar2021mpsn,bodnar2021cwn}: whether a higher-order interaction is \emph{compositional}, \emph{emergent}, or \emph{inhibitory} with respect to its observed or absent sets (Figure~\ref{fig:intro-comparison}). We call this property \emph{interaction compositionality}. Yet existing methods~\citep{yadati2019hypergcn,bai2021hcha,dong2020hnhn,qian2025adaptive}, which propagate messages only over observed hyperedges, leave it unmodeled. 
We argue that compositionality is the natural learning target for hypergraph learning. 
In this work, we accordingly change the learning object from supervising on observed hyperedges to supervising on the compositional, emergent, and inhibitory transition patterns. 

\begin{figure}[!t]
\centering
\includegraphics[width=\textwidth]{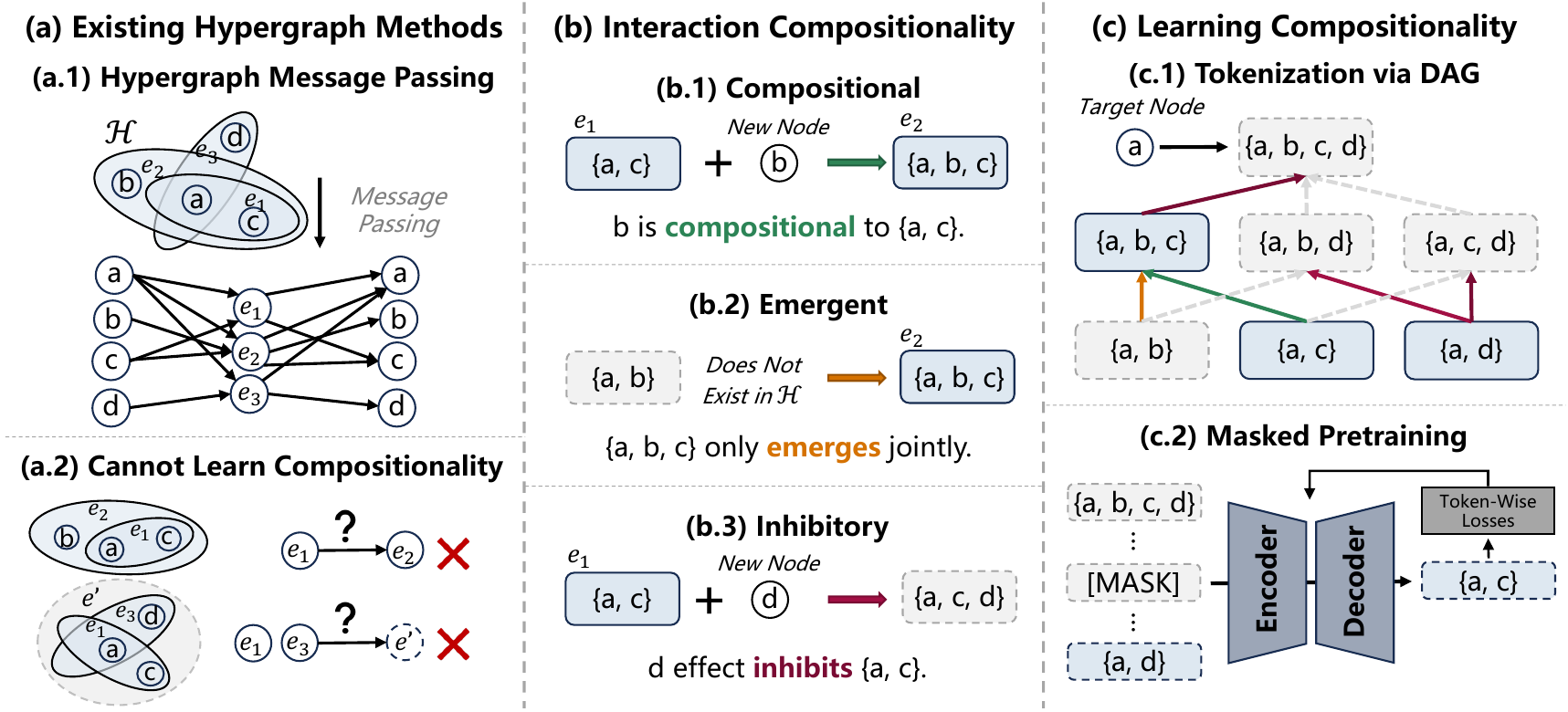}
\caption{\textbf{Interaction compositionality.} \textbf{(a)} Existing hypergraph methods propagate messages only over observed hyperedges (a.1), failing to distinguish hyperedge pairs whose compositionality structure differs (a.2). \textbf{(b)} Each adjacent-order subset pair falls into one of three regimes by which endpoints are observed: \emph{compositional} (b.1), \emph{emergent} (b.2), and \emph{inhibitory} (b.3). \textbf{(c)} HGPM tokenizes observed (solid) and unobserved (dashed) subsets containing a target node into an inclusion DAG with COMP/EMER/INHIB-labeled transitions (c.1, color-matched to (b)); the encoder is pretrained via masked-subset reconstruction with token-wise losses on subset semantics and existence (c.2).}
\vspace{-20pt}
\label{fig:intro-comparison}
\end{figure}

Polypharmacy~\citep{tatonetti2012ddi,chou2006synergism} makes this concrete. A clinician evaluating multi-drug regimens needs to distinguish three patterns. If both a pair and the full triple of drugs are effective, the third drug merely compounds an existing effect, and the prescription can be simplified. If only the full triple is effective while no smaller subset is, the three drugs are jointly required, and none can be dropped. If a pair is effective but adding the third drug eliminates the effect, the third drug actively disrupts the existing interaction and should be excluded. Each pattern points to a different prescription. 

Existing hypergraph methods, however, cannot distinguish these three patterns, for a structural reason. Hyperedge-centric methods such as AllSet~\citep{chien2022allset} and ED-HNN~\citep{wang2023edhnn} propagate between nodes and observed hyperedges; expansion-based methods such as HGNN~\citep{feng2019hgnn} and LEGCN~\citep{yang2022legcn} reduce the hypergraph to a clique-expanded pairwise graph and propagate over its edges; simplicial-complex networks~\citep{ebli2020scn,bodnar2021mpsn} pass messages between adjacent observed $k$-faces. All rely on message passing~\citep{telyatnikov2025lens}, where information flows only between entities in the hyperedges that already exist. Yet emergent and inhibitory regimes are defined by which subsets do \emph{not} exist, a signal that observed structure alone cannot carry. To a model that sees only what is observed, compositional, emergent, and inhibitory interactions collapse into a single binary prediction of whether the combination exists.

To prevent this collapse, we shift the paradigm: make compositionality itself the learning target. Capturing it requires addressing three challenges simultaneously. (1) Since absence is itself part of the compositional signal, both observed and absent subsets must become first-class entities the model can read. We tokenize subsets, each token carrying its content, order, and an explicit existence label that distinguishes observed from unobserved. (2) Each subset participates in many supersets, and each adjacent subset--superset transition carries its own compositional, emergent, or inhibitory label, so the representation must encode all of these multi-parent relations together. We organize tokens into an inclusion DAG whose subset--superset edges carry compositional, emergent, or inhibitory labels as first-class structure. (3) Subsets grow combinatorially with the size of the higher-order interaction (a 10-drug regimen has 1022 proper non-empty subsets, an 8-drug regimen 254), making the full DAG intractable on real datasets. We sample tokens within a bounded per-order budget, keeping the DAG local around each target while preserving the subset--superset edges that carry compositionality. Together, these three choices yield a representation in which every adjacent transition exposes the joint pattern of presence and absence that message passing leaves invisible.

We propose the Hypergraph Pattern Machine (HGPM), a Transformer that learns interaction compositionality from a target-centered inclusion DAG. For each target entity, HGPM tokenizes the subsets that contain it; each token carries an existence label, and each subset--superset edge carries one of three composition labels: compositional, emergent, or inhibitory. The encoder augments self-attention with structural biases~\citep{ying2021graphormer,kreuzer2021san,rampasek2022gps} that inject inclusion topology and composition labels directly, so the model reads the regime of each transition rather than re-learning it from token features. We pretrain HGPM by masking subset tokens~\citep{hou2022graphmae,xia2023molebert,hou2023graphmae2} and predicting whether each is observed; since each composition label is a deterministic function of subset existence, this prediction recovers compositionality on every adjacent edge of the DAG. On standard hypergraph node classification benchmarks, HGPM matches or exceeds state-of-the-art methods.

Notably, on the HODDI~\citep{wang2025hoddi} and JADER~\citep{pmda_jader} drug interaction datasets, HGPM is state-of-the-art under both edge classification and link prediction. Beyond benchmark scores, our case study reveals that HGPM correctly identifies the drug addition that inhibits a side effect among feature-identical candidates. We track a FOLFOX peripheral-neuropathy report in which two near-identical anti-EGFR/VEGF antibodies, panitumumab and bevacizumab, produce opposite regimes: panitumumab preserves the side effect (compositional) while bevacizumab inhibits it. HGPM correctly recovers this split, a discrimination existing similarity-based methods cannot make.

\section{Related Work}

\noindent\textbf{Message-Passing-Based Hypergraph Learning.}
Similar to graph learning~\citep{li2026substrate,zhao2023self,ju2022grape,qian2022co,zhao2021multi}, dominant approaches in hypergraph learning primarily rely on message passing over observed hyperedges. \emph{Hyperedge-centric methods}~\citep{yadati2019hypergcn,dong2020hnhn,huang2021unignn,chien2022allset,wang2023edhnn,ma2025hypergraph,ma2026bhygnn+} aggregate between nodes and hyperedges through incidence, with recent variants enriching propagation via sheaf, energy, framelet, or higher-order constructions~\citep{duta2023sheaf,wang2023phenomnn,li2025framehgnn,xie2025khgnn,sun2026healhgnn}. \emph{Expansion-based methods}~\citep{yang2022legcn,feng2019hgnn,ma2025adaptive} reduce the hypergraph to a pairwise graph and apply standard message passing~\citep{velickovic2018gat,kipf2017gcn}, while \emph{simplicial-complex networks}~\citep{bodnar2021mpsn,eijkelboom2023empsn} propagate over $k$-faces but require subset closure, foreclosing inhibitory transitions. Because all such schemes are message passing on observed structure~\citep{telyatnikov2025lens}, the joint pattern of subset presence and absence is not represented in the state. HGPM departs from this paradigm by tokenizing both observed and unobserved subsets, making that pattern an explicit input.

\noindent\textbf{Substructure Tokenization For Structured Data.}
A parallel line in graph learning replaces message passing with substructure-token sequences encoded by Transformers with structural attention biases~\citep{ying2021graphormer,rampasek2022gps} and pretrained by masked reconstruction~\citep{hou2022graphmae,xia2023molebert}; recent work establishes substructure tokenization as a scalable alternative on graphs~\citep{wang2024subgraphpool,wang2024gft,wang2025gpm,wang2025g2pm} and temporal graphs~\citep{ma2026tgpm}. HGPM extends it to hypergraphs: tokens shift from subgraph patterns to centered subset hierarchies, and supervision from token reconstruction to compositional, emergent, and inhibitory regimes. Concurrent hypergraph foundation models~\citep{gao2025hyperfm} pretrain a vertex-centric encoder for cross-corpus transfer, orthogonal to HGPM's subset-level focus.

\noindent\textbf{Higher-Order Drug Interaction Prediction.}
Computational modeling of drug interactions has historically been pairwise: synergy estimation~\citep{preuer2018deepsynergy,yang2021graphsynergy}, side-effect classification~\citep{zitnik2018decagon,ryu2018deepddi}, and substructure-based DDI~\citep{nyamabo2021ssiddi,chen2021muffin}. A more recent line extends to combinations of arbitrary size via permutation-invariant pooling~\citep{peng2019d3i}, latent combination types~\citep{nguyen2022sparse}, or hypergraph networks reduced to pairwise targets~\citep{saifuddin2023hygnn}, with HODDI~\citep{wang2025hoddi} and JADER~\citep{pmda_jader} supplying higher-order drug-effect data at scale. Both lines encode a combination as a single observation and score its effect, without contrasting it against the indicators of its subsets and adjacent supersets, the distinction HGPM is built to recover.

\section{Problem Setup}
\label{sec:problem-setup}

We consider a hypergraph $\mathcal{H} = (\mathcal{V}, \mathcal{E}, \mathbf{X})$, where $\mathcal{V}$ is a set of $n$ entities, $\mathcal{E} \subseteq 2^{\mathcal{V}}$ is the set of observed hyperedges, and $\mathbf{X} \in \mathbb{R}^{n \times d}$ collects per-entity features. For any subset $S \subseteq \mathcal{V}$, let $\mathds{1}_{S} = \mathds{1}[\, S \in \mathcal{E} \,] \in \{0, 1\}$ denote its observation indicator, with $\mathds{1}_{S} = 1$ when $S$ is an observed hyperedge and $0$ otherwise. Standard hypergraph methods supervise $f_{\theta}$ on downstream labels such as node class for $c \in \mathcal{V}$ or combination effect for $G \subseteq \mathcal{V}$; we instead supervise $f_{\theta}$ on the \emph{interaction compositionality} structure of $\mathcal{H}$, defined next.

\noindent\textbf{Interaction Compositionality.}
The \emph{inclusion hierarchy} of $\mathcal{H}$ is the set of subsets of $\mathcal{V}$ ordered by inclusion, $(2^{\mathcal{V}}, \subseteq)$. Within it, an \emph{adjacent-order pair} is any $(S, S')$ with $S \subset S'$ and $|S'| = |S| + 1$. Each adjacent-order pair has a \emph{composition label} determined by the indicators of $S$ and $S'$,
\begin{equation}
\label{eq:comp-label}
\mathrm{comp}(S, S') \;=\;
\begin{cases}
\textsc{Comp}, & \mathds{1}_{S} = 1,\; \mathds{1}_{S'} = 1, \\
\textsc{Emer}, & \mathds{1}_{S} = 0,\; \mathds{1}_{S'} = 1, \\
\textsc{Inhib}, & \mathds{1}_{S} = 1,\; \mathds{1}_{S'} = 0,
\end{cases}
\end{equation}
with pairs where both $S$ and $S'$ are absent excluded as uninformative. Collecting all labeled pairs gives the \emph{compositionality structure} $\mathcal{C}(\mathcal{H})$ of $\mathcal{H}$, which HGPM takes as the learning object: the model is trained to recover $\mathrm{comp}(S, S')$ from a representation of the inclusion hierarchy, so that supervision is keyed to adjacent-order transitions rather than to the indicator $\mathds{1}_{S}$ of a single subset in isolation.

\noindent\textbf{Why This Requires Going Beyond Message Passing.}
Recovering $\mathrm{comp}(S, S')$ requires reading the indicator pair $(\mathds{1}_S, \mathds{1}_{S'})$ jointly. This rules out hyperedge message-passing encoders, whose representations propagate only through observed hyperedges: there exist hypergraph pairs on which any such encoder produces identical features at every finite depth~\citep{xu2019gin,morris2019kgnn}, yet whose compositionality structures differ (Appendix~\ref{sec:thm-mp-separation}). What is needed instead is a representation that treats every relevant subset, observed or not, as a first-class object, with its existence indicator readable directly from the input. Such a representation provably separates pairs message passing cannot (Appendix~\ref{sec:thm-hgpm-separation}). We leave the theoretical grounding in Appendix~\ref{sec:theory}.

\section{Hypergraph Pattern Machine}
\label{sec:method}

We propose the HGPM, an inclusion-aware Transformer that learns interaction compositionality from a target-centered inclusion DAG of subset tokens (Figure~\ref{fig:framework}). 
Concretely, for each target $c \in \mathcal{V}$ HGPM proceeds in three stages: (i) build a bounded inclusion DAG $\mathcal{G}_c$ of subset tokens around $c$, with composition labels read directly from observation indicators (Section~\ref{sec:method-tokenization}); (ii) linearize $\mathcal{G}_c$ and encode it with self-attention augmented by pairwise biases that re-inject the inclusion topology and composition labels at every layer (Section~\ref{sec:method-encoder}); (iii) pretrain the encoder with masked-reconstruction losses, then transfer it via lightweight readouts. (Section~\ref{sec:method-pretrain})

\subsection{Compositional Tokenization}
\label{sec:method-tokenization}

\begin{figure}[t]
\centering
\includegraphics[width=\textwidth]{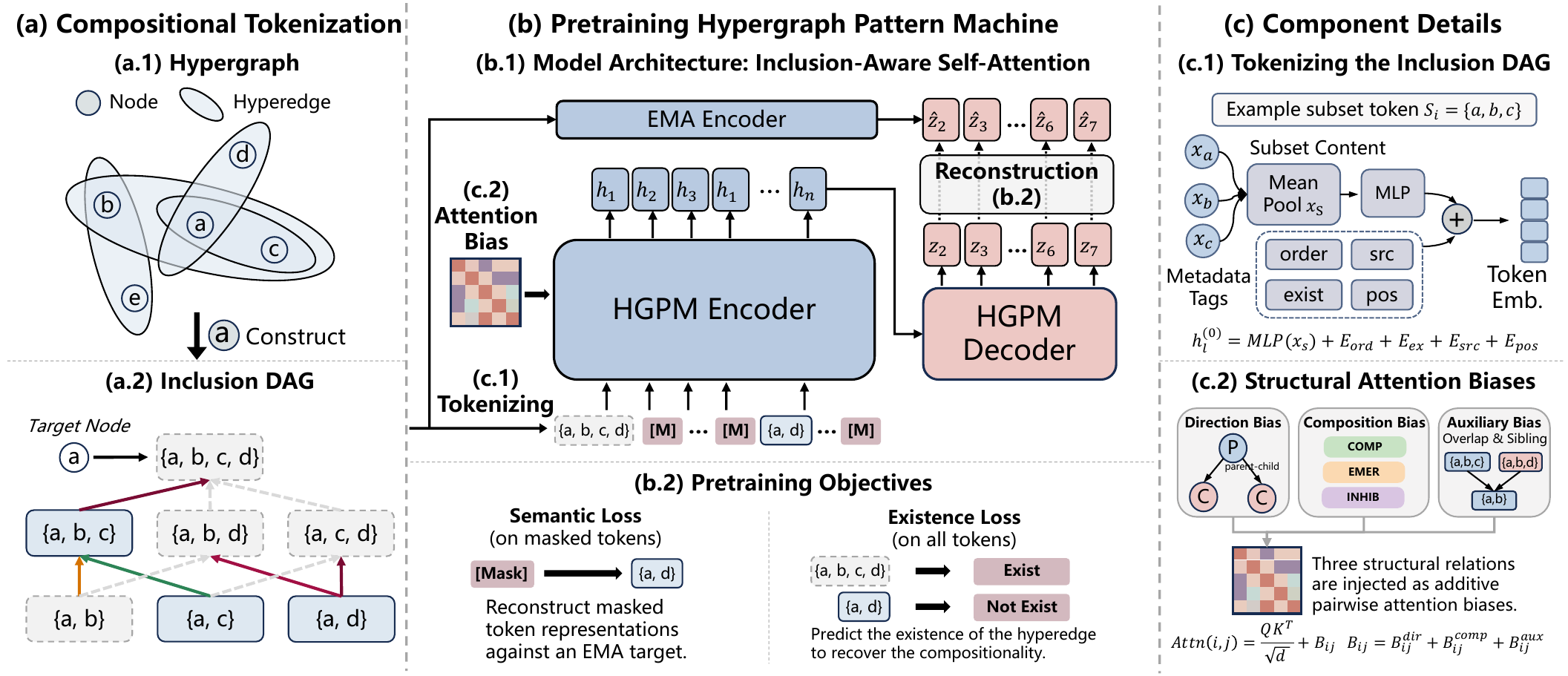}
\vspace{-15pt}
\caption{\textbf{Hypergraph Pattern Machine.} For each target entity, HGPM tokenizes subsets containing it and organizes them as an inclusion DAG with composition-labeled edges (left). A masked-reconstruction objective jointly supervises subset semantics and existence (middle). Pairwise structural biases inject inclusion topology and composition labels into self-attention (right).}
\vspace{-10pt}
\label{fig:framework}
\end{figure}

\noindent\textbf{Inclusion DAG.}
For each target entity $c \in \mathcal{V}$, HGPM organizes the local inclusion hierarchy around $c$ as a directed acyclic graph (DAG) $\mathcal{G}_c = (\mathcal{S}_c, \mathcal{R}_c)$ with existence-labeled nodes and composition-labeled edges, placing the compositionality structure $\mathcal{C}(\mathcal{H})$ directly in the input:
\begin{equation}
\label{eq:dag-sets}
\begin{aligned}
\mathcal{S}_c &= \{\, S \subseteq \mathcal{V} : c \in S,\; 1 \leq |S| \leq K_{\max} \,\}, \\
\mathcal{R}_c &= \{\, (S, S') \in \mathcal{S}_c \times \mathcal{S}_c : S \subset S',\; |S'| = |S| + 1 \,\}.
\end{aligned}
\end{equation}
The node set contains only subsets that include $c$, anchoring every token and edge to the prediction target. The edge set keeps only adjacent-order pairs, since compositionality is defined locally on $(S, S')$ with $|S'| = |S| + 1$ and non-adjacent inclusions arise as chains of atomic steps. Each edge carries the composition label $\mathrm{comp}(S, S')$, read off the data without learning, which promotes pair-level compositionality from a derived quantity to a first-class structural object. We use a DAG rather than a tree because a subset typically has multiple adjacent-order supersets carrying different labels. For example, $\{\mathrm{aspirin}, \mathrm{warfarin}\}$ extends to $\{\mathrm{aspirin}, \mathrm{warfarin}, \mathrm{ibuprofen}\}$ as COMP (additive bleeding risk~\citep{battistella2005warfarin}) but to $\{\mathrm{aspirin}, \mathrm{warfarin}, \mathrm{vitamin\ K}\}$ as INHIB (vitamin K reverses warfarin's effect~\citep{holbrook2005warfarin}), and a tree rooted at the pair must commit to one child label and erase the other. In practice, $|\mathcal{S}_c|$ grows combinatorially in $|\mathcal{V}|$. We therefore cap how many tokens we keep at each order, filling the cap with observed subsets first and negatives second (Appendix~\ref{sec:dag-construction}).

\noindent\textbf{Tokenizing the Inclusion DAG.}
To feed $\mathcal{G}_c$ to a Transformer, we first linearize it into a unique sequence by sorting subsets in descending $(|S|, S)$ order, placing the center singleton $\{c\}$ at the last position; the edge set $\mathcal{R}_c$ is dropped here and re-injected as structural attention biases in Section~\ref{sec:method-encoder}. Each position in the sequence is a token for one subset $S \in \mathcal{S}_c$, carrying four attributes (a mean-pooled semantic feature, the order, the observation indicator from Section~\ref{sec:problem-setup}, and a provenance tag):
\begin{equation}
\label{eq:token}
\mathbf{t}_S = \big(\mathbf{x}_S,\, |S|,\, \mathds{1}_{S},\, \tau(S)\big), \quad \mathbf{x}_S = \frac{1}{|S|}\sum_{v \in S} \mathbf{x}_v, \quad \tau(S) \in \{\mathrm{center},\, \mathrm{obs},\, \mathrm{neg}\}.
\end{equation}
Each attribute is chosen to expose a distinct signal that compositionality depends on. The \emph{semantic} feature $\mathbf{x}_S$ is mean-pooled for permutation-invariance~\citep{zaheer2017deepsets}. The \emph{order} $|S|$ lets the model condition on hierarchy depth. The \emph{existence indicator} $\mathds{1}_S$ is kept first-class because compositionality is defined by existence patterns alone, and so we can mask it during pretraining (Section~\ref{sec:method-pretrain}). The \emph{source} $\tau(S)$ tags each token as $\mathrm{center}$, $\mathrm{obs}$, or $\mathrm{neg}$. A $\mathrm{neg}$ token is an unobserved subset ($S \notin \mathcal{E}$), exposing absence as a signal. They are needed because EMER and INHIB transitions each have an absent endpoint. We construct them by \emph{local perturbation} of an observed hyperedge $e$ containing $c$: drop a non-center member, add a node from $\mathcal{V} \setminus e$, or replace a non-center member with one from $\mathcal{V} \setminus e$; retain only if order $\in [2, K_{\max}]$ and $S \notin \mathcal{E}$.

\subsection{Inclusion-Aware Self-Attention}
\label{sec:method-encoder}

The encoder is a Transformer over the linearized token sequence, with each token $S$ entering as $\mathbf{h}_S^{(0)} = W_{\mathrm{x}}(\mathbf{x}_S) + E_{\mathrm{ord}}[|S|] + E_{\mathrm{ex}}[\mathds{1}_S] + E_{\tau}[\tau(S)] + E_{\mathrm{pos}}[\pi(S)]$, a feature-MLP projection summed with learned lookups $E$ for the four token attributes and its linearized position $\pi(S)$~\citep{ying2021graphormer,rampasek2022gps}; the stacked embeddings form $\mathbf{H}_c^{(0)} \in \mathbb{R}^{|\mathcal{S}_c| \times d}$. Linearization, however, discards $\mathcal{R}_c$, so we re-inject inclusion topology and composition labels by augmenting self-attention with pairwise structural biases at every layer.

\noindent\textbf{Structural Attention Biases.}
At layer $\ell \in \{1, \ldots, L\}$, each self-attention layer augments the standard attention logit between tokens $i$ and $j$ with a learned pairwise bias~\citep{ying2021graphormer,rampasek2022gps,kreuzer2021san},
\begin{equation}
\label{eq:attn-bias}
a_{ij} = \frac{\mathbf{q}_i^{\top} \mathbf{k}_j}{\sqrt{d_h}} + b^{\mathrm{dir}}_{ij} + b^{\mathrm{comp}}_{ij} + b^{\mathrm{aux}}_{ij},
\end{equation}
where $\mathbf{q}_i, \mathbf{k}_j \in \mathbb{R}^{d_h}$ are the per-head query/key projections of the layer-$\ell$ token states $\mathbf{h}^{(\ell)}_i, \mathbf{h}^{(\ell)}_j$ ($d_h$ is the head dimension). The \emph{direction} bias $b^{\mathrm{dir}}_{ij}$ encodes the inclusion edge between $S_i$ and $S_j$, taking one of three values (no edge, $S_i$ is a parent of $S_j$, or $S_j$ is a parent of $S_i$), making the topology of $\mathcal{R}_c$ visible to attention. The \emph{composition} bias $b^{\mathrm{comp}}_{ij}$ is indexed by the joint existence roles $(\tau(S_i), \tau(S_j))$; on adjacent-order pairs this directly identifies the composition label $\mathrm{comp}(S_i, S_j) \in \{\textsc{Comp}, \textsc{Emer}, \textsc{Inhib}\}$, providing the channel through which compositionality enters attention. The auxiliary term $b^{\mathrm{aux}}_{ij}$ folds in non-load-bearing structural priors (order, overlap, sibling) detailed in Appendix~\ref{sec:aux-biases}. All biases are learned, head-specific, and added at every layer, yielding contextualized representations $\mathbf{H}_c^{(L)} \in \mathbb{R}^{|\mathcal{S}_c| \times d}$ after $L$ stacked layers.

\subsection{Compositionality-Aware Pretraining}
\label{sec:method-pretrain}

We pretrain the encoder with a masked reconstruction objective~\citep{hou2022graphmae,xia2023molebert,hou2023graphmae2} that jointly optimizes two token-wise losses, turning the compositional labels (Eq.~\eqref{eq:comp-label}) into an explicit training signal. For each inclusion DAG $\mathcal{G}_c$, we randomly mask a fraction $\rho$ of non-center tokens $\mathcal{M} \subset \mathcal{S}_c \setminus \{\{c\}\}$ by replacing their input embedding with a learnable vector $\mathbf{e}_{\mathrm{mask}}$ and forward the corrupted sequence through the encoder, yielding contextualized representations $\tilde{\mathbf{h}}_S^{(L)}$. Two lightweight prediction heads on $\tilde{\mathbf{h}}_S^{(L)}$ produce $\hat{\mathds{1}}_S, \hat{\mathbf{z}}_S$, while an EMA teacher~\citep{grill2020byol,caron2021dino,baevski2022data2vec} produces the regression target $\mathbf{z}_S = \mathrm{Teacher}(\mathbf{x}_S)$. The \emph{semantic} loss $\mathcal{L}_{\mathrm{sem}}$ regresses each masked-token prediction against the EMA-teacher target. The \emph{existence} loss $\mathcal{L}_{\mathrm{exist}}$ supervises $\mathds{1}_S$ at every token. The total objective combines both:
\begin{equation}
\label{eq:pretraining}
\mathcal{L} = \underbrace{\tfrac{1}{|\mathcal{M}|}\!\!\sum_{S \in \mathcal{M}}\!\! \|\hat{\mathbf{z}}_S - \mathbf{z}_S\|_2^2}_{\mathcal{L}_{\mathrm{sem}}} + \lambda_{\mathrm{exist}}\!\underbrace{\tfrac{1}{|\mathcal{S}_c|}\!\!\sum_{S \in \mathcal{S}_c}\!\!\mathrm{BCE}(\hat{\mathds{1}}_S, \mathds{1}_S)}_{\mathcal{L}_{\mathrm{exist}}}.
\end{equation}
The pretrained encoder transfers to node-level and edge-level downstream tasks via lightweight MLP readouts and end-to-end finetuning, with full readout definitions in Appendix~\ref{sec:downstream-readouts}.
\section{Generality on Hypergraph Benchmarks}
\label{sec:exp-general}


\providecolor{rkbest}{HTML}{B91C2C}
\providecolor{rksecond}{HTML}{1F4E8F}
\providecolor{rkthird}{HTML}{2A7A3F}
\providecommand{\rkone}[1]{{\boldmath\textcolor{rkbest}{#1}}}
\providecommand{\rktwo}[1]{{\boldmath\textcolor{rksecond}{#1}}}
\providecommand{\rkthree}[1]{{\boldmath\textcolor{rkthird}{#1}}}

\begin{table}[!t]
\caption{\textbf{Hypergraph node classification.} Test accuracy (\%, mean$_{\pm\text{std}}$) over multiple seeds on eight benchmarks, split into homophilic (left) and heterophilic (right). AR denotes each method's mean rank across the datasets (lower is better). \rkone{Best} / \rktwo{second} / \rkthree{third} per column. }
\label{tab:node-classification}
\centering
\small
\setlength{\tabcolsep}{3pt}
\resizebox{\textwidth}{!}{%
\begin{tabular}{l cccc | cccc | c}
\toprule
& \multicolumn{4}{c|}{\textbf{Homophilic}} & \multicolumn{4}{c|}{\textbf{Heterophilic}} & \textbf{AR}\,$\downarrow$ \\
\cmidrule(lr){2-5} \cmidrule(lr){6-9}
Method & Citeseer & Pubmed & Cora-CA & DBLP-CA & Congress & Senate & Walmart & House & \\
\midrule
MLP                                                & $72.7_{\pm 1.6}$ & $87.5_{\pm 0.5}$ & $74.3_{\pm 1.9}$ & $84.8_{\pm 0.2}$ & --                & --                & $45.5_{\pm 0.2}$ & $67.9_{\pm 2.3}$ & $12.5$ \\
CEGAT~\citep{chien2022allset}                      & $70.6_{\pm 1.3}$ & $86.8_{\pm 0.4}$ & $76.2_{\pm 1.2}$ & $88.6_{\pm 0.3}$ & --                & --                & $51.1_{\pm 0.6}$ & $69.1_{\pm 3.0}$ & $12.8$ \\
\midrule
HGNN~\citep{feng2019hgnn}                          & $72.4_{\pm 1.2}$ & $86.4_{\pm 0.4}$ & $82.6_{\pm 1.7}$ & $91.0_{\pm 0.2}$ & $91.3_{\pm 1.2}$ & $48.6_{\pm 4.5}$ & $62.0_{\pm 0.2}$ & $61.4_{\pm 3.0}$ & $11.6$ \\
HyperGCN~\citep{yadati2019hypergcn}                & $71.3_{\pm 0.8}$ & $82.8_{\pm 8.7}$ & $79.5_{\pm 2.1}$ & $89.4_{\pm 0.3}$ & $55.1_{\pm 2.0}$ & $42.5_{\pm 3.7}$ & $44.7_{\pm 2.8}$ & $48.3_{\pm 2.9}$ & $13.6$ \\
HNHN~\citep{dong2020hnhn}                          & $72.6_{\pm 1.6}$ & $86.9_{\pm 0.3}$ & $77.2_{\pm 1.5}$ & $86.8_{\pm 0.3}$ & $53.4_{\pm 1.5}$ & $50.9_{\pm 6.3}$ & $47.2_{\pm 0.4}$ & $67.8_{\pm 2.6}$ & $12.4$ \\
UniGCNII~\citep{huang2021unignn}                   & $73.1_{\pm 2.2}$ & $88.3_{\pm 0.4}$ & $83.6_{\pm 1.1}$ & $91.7_{\pm 0.2}$ & \rktwo{$94.8_{\pm 0.8}$} & $49.3_{\pm 4.3}$ & $54.5_{\pm 0.4}$ & $67.3_{\pm 2.6}$ & $8.4$ \\
AllDeepSets~\citep{chien2022allset}                & $70.8_{\pm 1.6}$ & $88.8_{\pm 0.3}$ & $82.0_{\pm 1.5}$ & $91.3_{\pm 0.3}$ & $91.8_{\pm 1.5}$ & $48.2_{\pm 5.7}$ & $64.6_{\pm 0.3}$ & $67.8_{\pm 2.4}$ & $9.8$ \\
AllSetTransformer~\citep{chien2022allset}          & $73.1_{\pm 1.2}$ & $88.7_{\pm 0.4}$ & $83.6_{\pm 1.5}$ & $91.5_{\pm 0.2}$ & $92.2_{\pm 1.1}$ & $51.8_{\pm 5.2}$ & $65.5_{\pm 0.3}$ & $69.3_{\pm 2.2}$ & $7.3$ \\
ED-HNN~\citep{wang2023edhnn}                       & $73.7_{\pm 1.4}$ & \rktwo{$89.0_{\pm 0.5}$} & $84.0_{\pm 1.6}$ & $91.9_{\pm 0.2}$ & \rkone{$95.0_{\pm 1.0}$} & $64.8_{\pm 5.1}$ & $66.9_{\pm 0.4}$ & $72.5_{\pm 2.3}$ & \rkthree{$4.9$} \\
\midrule
SheafHyperGNN~\citep{duta2023sheaf}                & $74.7_{\pm 1.2}$ & $87.7_{\pm 0.6}$ & \rktwo{$85.5_{\pm 1.3}$} & $91.6_{\pm 0.2}$ & $91.8_{\pm 1.6}$ & $68.7_{\pm 4.7}$ & --                & $73.6_{\pm 2.3}$ & $6.0$ \\
PhenomNN~\citep{wang2023phenomnn}                  & \rktwo{$75.1_{\pm 1.6}$} & $88.1_{\pm 0.5}$ & \rkone{$85.8_{\pm 0.9}$} & \rkthree{$91.9_{\pm 0.2}$} & $88.2_{\pm 1.5}$ & $67.7_{\pm 5.2}$ & $63.4_{\pm 0.4}$ & $70.7_{\pm 2.4}$ & $6.0$ \\
FrameHGNN~\citep{li2025framehgnn}                  & $74.7_{\pm 2.1}$ & $88.7_{\pm 0.4}$ & \rkthree{$85.2_{\pm 0.7}$} & --                & --                & $67.6_{\pm 5.3}$ & --                & $72.8_{\pm 2.2}$ & $5.2$ \\
KHGNN~\citep{xie2025khgnn}                         & $74.8_{\pm 1.1}$ & $88.5_{\pm 0.5}$ & --                & $91.2_{\pm 0.4}$ & $92.6_{\pm 1.0}$ & $71.2_{\pm 3.9}$ & $65.2_{\pm 0.8}$ & $74.3_{\pm 2.9}$ & $5.6$ \\
HealHGNN~\citep{sun2026healhgnn}                   & \rkthree{$75.1_{\pm 1.2}$} & \rkthree{$88.8_{\pm 0.3}$} & --                & \rktwo{$92.0_{\pm 0.3}$} & $93.5_{\pm 0.8}$ & \rktwo{$76.1_{\pm 4.1}$} & \rkthree{$68.2_{\pm 0.6}$} & \rkthree{$77.2_{\pm 2.3}$} & \rktwo{$2.9$} \\
\midrule
HGPM (no pretrain)                                 & $74.1_{\pm 1.1}$ & $87.5_{\pm 0.5}$ & $82.3_{\pm 1.7}$ & $91.1_{\pm 0.3}$ & $92.4_{\pm 0.6}$ & \rkthree{$75.2_{\pm 3.5}$} & \rktwo{$68.9_{\pm 0.3}$} & \rktwo{$77.8_{\pm 1.2}$} & $6.4$ \\
\textbf{HGPM}                                      & \rkone{$77.1_{\pm 1.2}$} & \rkone{$89.7_{\pm 0.5}$} & $85.1_{\pm 0.2}$ & \rkone{$92.2_{\pm 0.2}$} & \rkthree{$94.1_{\pm 1.0}$} & \rkone{$77.2_{\pm 4.8}$} & \rkone{$71.2_{\pm 0.5}$} & \rkone{$79.2_{\pm 1.8}$} & \rkone{$1.6$} \\
\bottomrule
\end{tabular}%
}
\vspace{-10pt}
\end{table}

\subsection{Setup}

We evaluate on eight standard hypergraph node classification benchmarks, following the suite established by \citet{chien2022allset} and \citet{wang2023edhnn}: four homophilic (Citeseer, Pubmed, Cora-CA, DBLP-CA) and four heterophilic (Congress, Senate, Walmart, House) benchmarks spanning co-citation, co-authorship, co-purchase, and political-collaboration domains with 50/25/25 splits (detailed in Appendix~\ref{sec:dataset-details}). We compare HGPM against a comprehensive set of models: MLP, CEGAT~\citep{chien2022allset}, HGNN~\citep{feng2019hgnn}, HyperGCN~\citep{yadati2019hypergcn}, HNHN~\citep{dong2020hnhn}, UniGCNII~\citep{huang2021unignn}, AllSet~\citep{chien2022allset}, ED-HNN~\citep{wang2023edhnn}), SheafHyperGNN~\citep{duta2023sheaf}, PhenomNN~\citep{wang2023phenomnn}, FrameHGNN~\citep{li2025framehgnn}, KHGNN~\citep{xie2025khgnn}, and HealHGNN~\citep{sun2026healhgnn}. For HGPM we report both \emph{HGPM (no pretrain)} trained from scratch and the pretrained-then-finetuned variant. HGPM hyperparameters are tuned by per-dataset random search; the details are presented in Appendix~\ref{sec:hyperparams}. All results are reported as mean test accuracy over ten random seeds. Despite operating on subset-token sequences, HGPM trains efficiently on a single A40 GPU: on smaller benchmarks (e.g., Cora-CA, Citeseer, Senate), pretraining completes within 2 hours and finetuning within 10 minutes; on larger benchmarks (e.g., Walmart), the full pretrain-and-finetune pipeline completes within 5 hours.

\subsection{Main Results}

Table~\ref{tab:node-classification} reports node classification accuracy on the eight benchmarks. HGPM achieves the best test accuracy on six of eight datasets and an average rank of $1.6$, well ahead of the second-strongest baseline (HealHGNN, AR $= 2.9$). The gap is largest on the heterophilic benchmarks (Senate, Walmart, House), where message-passing methods structurally struggle, indicating that the inclusion DAG representation transfers cleanly to hypergraphs whose subset structure deviates from clean homophily. Pretraining contributes meaningfully: removing it drops HGPM from AR $1.6$ to AR $6.4$, putting it on par with recent state-of-the-art encoders, which confirms that the masked subset reconstruction objective extracts compositional signal beyond what scratch training provides. Together, these results show that the inclusion DAG tokenizer is a general-purpose hypergraph representation.

\subsection{Model Spaces and Insights}
\label{sec:model-space}

HGPM exposes six design axes (subset tokens, inclusion DAG edges, the token encoder, structural attention biases, the pretraining reconstruction target, and token sequence order), jointly summarized in Table~\ref{tab:ablation-model-space}; per-axis numerical discussion is deferred to Appendix~\ref{sec:full-ablation}. 

\begin{table}[t]
\centering
\caption{\textbf{Ablation on the eight hypergraph benchmarks}, reported as homophilic and heterophilic group averages. Gray marks the HGPM default; full per-dataset breakdown in Appendix~\ref{sec:full-ablation}.}
\label{tab:ablation-model-space}
\footnotesize
\setlength{\tabcolsep}{4pt}

\begin{subtable}[t]{0.32\textwidth}
\centering
\caption{Subset Tokens}
\label{tab:ab-tokens}
\begin{tabular*}{\linewidth}{@{}l@{\extracolsep{\fill}}rr@{}}
\toprule
Setting & Homo & Hetero \\
\midrule
Observed hyperedges & $83.43$ & $75.46$ \\
Uncentered random   & $84.78$ & $77.51$ \\
No negatives        & $85.81$ & $79.61$ \\
Inclusion DAG       & \cellcolor{gray!20}$\mathbf{86.03}$ & \cellcolor{gray!20}$\mathbf{80.42}$ \\
\bottomrule
\end{tabular*}
\end{subtable}\hfill
\begin{subtable}[t]{0.32\textwidth}
\centering
\caption{Inclusion DAG Edges}
\label{tab:ab-edges}
\begin{tabular*}{\linewidth}{@{}l@{\extracolsep{\fill}}rr@{}}
\toprule
Setting & Homo & Hetero \\
\midrule
No edges            & $83.66$ & $76.66$ \\
Random edges        & $84.58$ & $78.15$ \\
Center-only         & $85.31$ & $79.23$ \\
Adjacent inclusion  & \cellcolor{gray!20}$\mathbf{86.03}$ & \cellcolor{gray!20}$\mathbf{80.42}$ \\
\bottomrule
\end{tabular*}
\end{subtable}\hfill
\begin{subtable}[t]{0.32\textwidth}
\centering
\caption{Token Encoder}
\label{tab:ab-encoder}
\begin{tabular*}{\linewidth}{@{}l@{\extracolsep{\fill}}rr@{}}
\toprule
Setting & Homo & Hetero \\
\midrule
Mean pooling        & $83.29$ & $76.29$ \\
GRU                 & $79.55$ & $72.91$ \\
Vanilla Transformer & $84.78$ & $77.99$ \\
Inclusion-aware     & \cellcolor{gray!20}$\mathbf{86.03}$ & \cellcolor{gray!20}$\mathbf{80.42}$ \\
\bottomrule
\end{tabular*}
\end{subtable}

\vspace{-5pt}

\begin{subtable}[t]{0.32\textwidth}
\centering
\caption{Structural Attention Biases}
\label{tab:ab-bias}
\begin{tabular*}{\linewidth}{@{}l@{\extracolsep{\fill}}rr@{}}
\toprule
Setting & Homo & Hetero \\
\midrule
No bias             & $84.78$ & $77.99$ \\
Composition only    & $85.71$ & $79.85$ \\
Direction $+$ aux   & $85.30$ & $79.10$ \\
Full bias           & \cellcolor{gray!20}$\mathbf{86.03}$ & \cellcolor{gray!20}$\mathbf{80.42}$ \\
\bottomrule
\end{tabular*}
\end{subtable}\hfill
\begin{subtable}[t]{0.32\textwidth}
\centering
\caption{Reconstruction Target}
\label{tab:ab-target}
\begin{tabular*}{\linewidth}{@{}l@{\extracolsep{\fill}}rr@{}}
\toprule
Setting & Homo & Hetero \\
\midrule
Existence label     & $83.34$ & $78.15$ \\
Raw feature         & $84.51$ & $79.13$ \\
Composition label   & $85.18$ & $79.77$ \\
TeacherMLP target   & \cellcolor{gray!20}$\mathbf{86.03}$ & \cellcolor{gray!20}$\mathbf{80.42}$ \\
\bottomrule
\end{tabular*}
\end{subtable}\hfill
\begin{subtable}[t]{0.32\textwidth}
\centering
\caption{Token Sequence Order}
\label{tab:ab-order}
\begin{tabular*}{\linewidth}{@{}l@{\extracolsep{\fill}}rr@{}}
\toprule
Setting & Homo & Hetero \\
\midrule
Ascending           & $85.31$ & $79.08$ \\
Shuffled within     & $85.64$ & $79.67$ \\
Fully shuffled      & $85.89$ & $80.16$ \\
Descending          & \cellcolor{gray!20}$\mathbf{86.03}$ & \cellcolor{gray!20}$\mathbf{80.42}$ \\
\bottomrule
\end{tabular*}
\end{subtable}
\vspace{-15pt}
\end{table}

\textbf{Insight 1: Compositional structure must be explicit, not derived.} Three independent architectural layers, namely token construction (Table~\ref{tab:ab-tokens}), edge selection (Table~\ref{tab:ab-edges}), and attention bias (Table~\ref{tab:ab-bias}), converge on the same conclusion: the COMP/EMER/INHIB regimes must be injected into the input rather than recovered by the encoder. Centered subset tokens beat raw hyperedges most decisively on heterophilic data; center-anchored inclusion edges substantially outperform random edges of identical count, so the \emph{content} of the inclusion DAG matters far more than its density; and the composition bias $b^{\mathrm{comp}}$ in Eq.~\eqref{eq:attn-bias} alone recovers about 75\% of the full-bias gap on every benchmark, the most robust finding in the design space. Linearization probes (Table~\ref{tab:ab-order}) reverse-confirm this picture: topological signal flows through the attention bias rather than the position embedding~\citep{ying2021graphormer,rampasek2022gps}, so compositionality is a bias-side property and not a sequence-side one.

\textbf{Insight 2: Inductive bias alignment matters more than model capacity.} A \emph{wrong} inductive bias is worse than no token-to-token interaction at all, and the same pattern recurs at two independent layers. In the subset-token encoder (Table~\ref{tab:ab-encoder}), the GRU baseline underperforms even mean pooling on every dataset because it imposes a sequential ordering on a permutation-invariant subset-token set~\citep{zaheer2017deepsets}; the inclusion-aware Transformer's advantage over the vanilla Transformer is therefore inclusion-aware refinement, not extra capacity. The edge-selection ablation (Table~\ref{tab:ab-edges}) shows the same dynamic from the opposite side: random edges matched in count to the adjacent inclusion edges substantially underperform center-anchored ones, so capacity-matched but bias-misaligned wiring leaves most of the signal on the table. Across the four architectural axes (a)--(d), the heterophilic benchmarks see roughly $1.6$--$2\times$ the impact of the homophilic ones, confirming that HGPM's value is largest precisely where message-passing methods structurally fail.

\textbf{Insight 3: Pretraining target richness unlocks the pretrain gap.} A misspecified pretraining target is strictly worse than no pretraining at all (Table~\ref{tab:ab-target}). Existence-label reconstruction $\mathcal{L}_{\mathrm{exist}}$ alone yields negative transfer on every benchmark, while only the TeacherMLP target $\mathcal{L}_{\mathrm{sem}}$ recovers the full pretrain-vs-no-pretrain gap; raw subset features and discrete composition labels lie in between. Critically, the ordering existence $\to$ raw feature $\to$ composition label $\to$ TeacherMLP is monotone on every benchmark, ruling out dataset-specific noise and validating $\mathcal{L}_{\mathrm{sem}}$ as the primary pretraining loss. The mechanism is that continuous, high-level semantic alignment~\citep{baevski2022data2vec,grill2020byol,hou2022graphmae} carries information beyond the saturation point of any discrete reconstruction target: each step in the hierarchy adds richness the previous step could not encode.

\section{Compositionality on Drug Interaction}
\label{sec:exp-drugs}

\subsection{Setup}

We evaluate HGPM on two pharmacovigilance ecosystems whose reporting conventions, drug vocabularies, and label distributions differ substantially, so that strong cross-corpus performance reflects robustness to the heterogeneity real adverse-event surveillance encounters across jurisdictions. HODDI~\citep{wang2025hoddi}, derived from the U.S. FAERS, covers $1{,}821$ drugs and a hierarchical $1{,}710$-class side-effect ontology with dense per-drug supervision; JADER~\citep{pmda_jader}, the Japanese Adverse Drug Event Report database maintained by Japan's PMDA, covers $4{,}230$ drugs over a more concentrated label space of $100$ retained classes (per-dataset statistics in Appendix~\ref{sec:dataset-drugs}). Each dataset supports two tasks: \emph{edge classification}, predicting the dominant side-effect class of an observed drug combination, and \emph{link prediction}, a binary task distinguishing observed combinations from constructed negatives. Both use $50/25/25$ random train/validation/test splits. We compare HGPM against baselines spanning the methodological spectrum currently applied to drug interaction modeling: flat encoders that aggregate regimens as bag-of-drugs (MLP, GAT~\citep{velickovic2018gat}, GCN~\citep{kipf2017gcn}) and hypergraph encoders that treat each regimen as an atomic hyperedge (HGNN~\citep{feng2019hgnn}, AllSetTransformer~\citep{chien2022allset}, ED-HNN~\citep{wang2023edhnn}, KHGNN~\citep{xie2025khgnn}). For edge classification we report F1 and AUROC; for link prediction we report AUROC and AUPRC. All metrics are averaged over ten random seeds.


\providecolor{rkbest}{HTML}{B91C2C}
\providecolor{rksecond}{HTML}{1F4E8F}
\providecolor{rkthird}{HTML}{2A7A3F}
\providecommand{\rkone}[1]{{\boldmath\textcolor{rkbest}{#1}}}
\providecommand{\rktwo}[1]{{\boldmath\textcolor{rksecond}{#1}}}
\providecommand{\rkthree}[1]{{\boldmath\textcolor{rkthird}{#1}}}

\begin{table}[!t]
\caption{\textbf{Drug interaction prediction.} Test performance (\%, mean$_{\pm\text{std}}$) on HODDI and JADER, each with edge classification and binary link prediction. AR is the mean rank across the eight metric columns (lower is better). \rkone{Best} / \rktwo{second} / \rkthree{third} per column.}
\label{tab:drug-benchmark}
\centering
\small
\setlength{\tabcolsep}{2.5pt}
\resizebox{\textwidth}{!}{%
\begin{tabular}{l cc cc | cc cc | c}
\toprule
& \multicolumn{4}{c|}{\textbf{HODDI}} & \multicolumn{4}{c|}{\textbf{JADER}} & \\
\cmidrule(lr){2-5} \cmidrule(lr){6-9}
& \multicolumn{2}{c}{Edge Class.} & \multicolumn{2}{c|}{Link Pred.} & \multicolumn{2}{c}{Edge Class.} & \multicolumn{2}{c|}{Link Pred.} & \textbf{AR}\,$\downarrow$ \\
\cmidrule(lr){2-3} \cmidrule(lr){4-5} \cmidrule(lr){6-7} \cmidrule(lr){8-9}
Method & F1 & AUROC & AUROC & AUPRC & F1 & AUROC & AUROC & AUPRC & \\
\midrule
MLP                                                & $88.3_{\pm 0.1}$ & $93.2_{\pm 0.1}$ & $96.6_{\pm 0.1}$ & $97.3_{\pm 0.1}$ & $74.8_{\pm 0.3}$ & $82.3_{\pm 0.0}$ & $57.9_{\pm 0.1}$ & $56.4_{\pm 0.3}$ & $5.8$ \\
GAT~\citep{velickovic2018gat}                      & $88.3_{\pm 0.3}$ & $93.3_{\pm 0.3}$ & $95.3_{\pm 0.1}$ & $95.9_{\pm 0.1}$ & $72.1_{\pm 0.6}$ & $79.3_{\pm 0.1}$ & $56.3_{\pm 0.2}$ & $55.1_{\pm 0.3}$ & $7.8$ \\
GCN~\citep{kipf2017gcn}                            & $89.5_{\pm 0.1}$ & $94.9_{\pm 0.0}$ & $96.1_{\pm 0.1}$ & $97.0_{\pm 0.1}$ & $73.8_{\pm 0.2}$ & $81.1_{\pm 0.2}$ & $57.8_{\pm 0.0}$ & $56.3_{\pm 0.1}$ & $6.1$ \\
\midrule
HGNN~\citep{feng2019hgnn}                          & $89.5_{\pm 0.2}$ & $94.9_{\pm 0.1}$ & $96.1_{\pm 0.1}$ & $97.0_{\pm 0.1}$ & $74.3_{\pm 0.4}$ & $81.2_{\pm 0.2}$ & $57.7_{\pm 0.2}$ & $56.2_{\pm 0.1}$ & $6.1$ \\
AllSetTransformer~\citep{chien2022allset}          & \rkthree{$90.4_{\pm 0.1}$} & $94.9_{\pm 0.0}$ & \rkthree{$96.9_{\pm 0.1}$} & $97.6_{\pm 0.1}$ & \rkthree{$75.4_{\pm 0.3}$} & $83.4_{\pm 0.1}$ & \rktwo{$59.6_{\pm 0.1}$} & $57.5_{\pm 0.1}$ & $3.6$ \\
ED-HNN~\citep{wang2023edhnn}                       & $90.2_{\pm 0.1}$ & \rkthree{$95.6_{\pm 0.0}$} & $96.9_{\pm 0.1}$ & \rktwo{$97.8_{\pm 0.1}$} & $75.3_{\pm 0.2}$ & \rkthree{$83.5_{\pm 0.1}$} & $59.1_{\pm 0.1}$ & \rktwo{$58.2_{\pm 0.1}$} & \rkthree{$3.3$} \\
KHGNN~\citep{xie2025khgnn}                         & \rktwo{$90.5_{\pm 0.1}$} & \rktwo{$95.7_{\pm 0.0}$} & \rktwo{$97.2_{\pm 0.1}$} & \rkthree{$97.6_{\pm 0.1}$} & \rktwo{$75.7_{\pm 0.3}$} & \rktwo{$83.7_{\pm 0.1}$} & \rkthree{$59.4_{\pm 0.1}$} & \rkthree{$57.9_{\pm 0.1}$} & \rktwo{$2.4$} \\
\midrule
\textbf{HGPM}                                      & \rkone{$92.9_{\pm 0.1}$} & \rkone{$96.8_{\pm 0.0}$} & \rkone{$97.9_{\pm 0.0}$} & \rkone{$98.2_{\pm 0.0}$} & \rkone{$78.4_{\pm 0.2}$} & \rkone{$86.2_{\pm 0.1}$} & \rkone{$63.2_{\pm 0.1}$} & \rkone{$61.5_{\pm 0.1}$} & \rkone{$1.0$} \\
\bottomrule
\end{tabular}%
}
\vspace{-15pt}
\end{table}

\subsection{Main Results}

Table~\ref{tab:drug-benchmark} reports two clinically distinct queries: \emph{edge classification} (post-market triage: which adverse-effect category dominates an observed regimen) and \emph{link prediction} (screening: whether a candidate combination is safety-relevant at all). HGPM achieves the best on all eight metric columns (AR $= 1.0$, ahead of KHGNN at $2.4$ and ED-HNN at $3.3$). On the saturated HODDI benchmark, HGPM advances edge F1 from $90.5$ to $92.9$; on the harder JADER, whose drug vocabulary more than doubles ($4{,}230$ vs.\ $1{,}821$), link AUROC rises from $59.6$ to $63.2$. Gains track regime difficulty rather than concentrating on saturated supervision, suggesting the inductive bias is doing real work. The reason is pharmacological: the adverse-event profile of a $k$-drug regimen is not the union of its members' single-drug profiles, since pairwise sub-regimens within it carry distinct synergy, antagonism, or dose-modulation signatures~\citep{chou2006synergism,tatonetti2012ddi} that the full regimen inherits. Flat hyperedge and clique-expansion encoders collapse this internal structure into one object; HGPM's inclusion DAG keeps every observed lower-order sub-regimen as a first-class informant, mirroring the compositional reasoning clinicians apply to novel polypharmacy.

\begin{wrapfigure}{r}{0.5\linewidth}
\centering
\vspace{-1.2em}
\includegraphics[width=\linewidth]{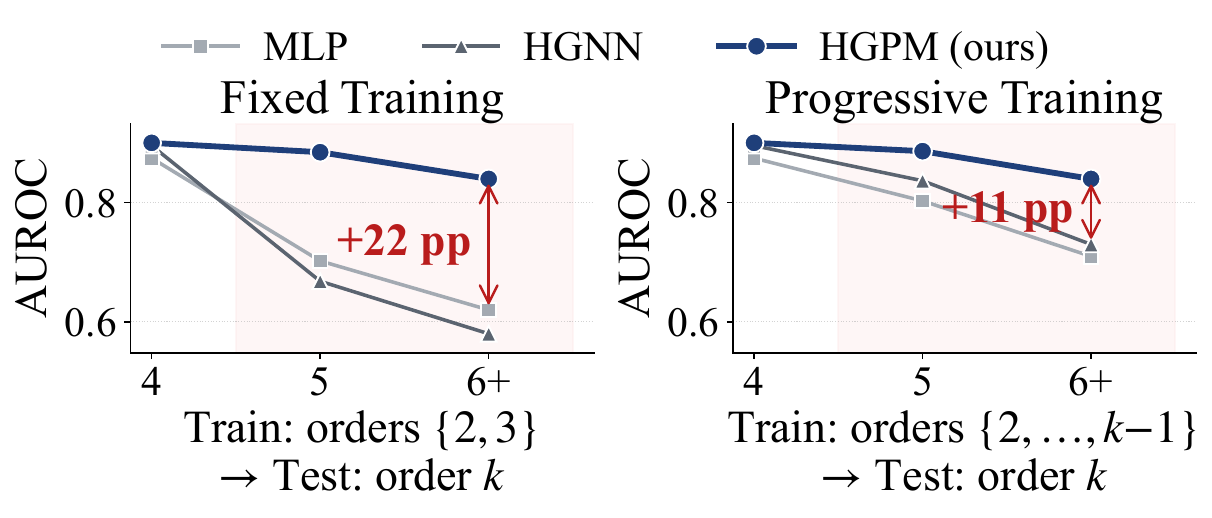}
\caption{\textbf{Cross-order generalization on HODDI.} We report the test AUROC for \emph{Fixed Training} on orders $\{2, 3\}$ and \emph{Progressive Training} on $\{2, \ldots, k{-}1\}$ for each test order $k$.}
\label{fig:cross-order}
\vspace{-1em}
\end{wrapfigure}

\noindent\textbf{Cross-Order Generalization.}
Higher-order drug interaction data is intrinsically heavy-tailed: the combinatorial space grows as $\binom{|\mathcal{V}|}{k}$ and reporting bias toward common regimens leaves the tail sparse. HODDI, for instance, contains roughly $49{,}000$ size-$2$ combinations but only $449$ size-$8$ regimens. The clinically actionable prescriptions are exactly the rare ones (novel polypharmacy regimens, off-label combinations, post-marketing surveillance leads), so cross-order generalization is the deployment scenario itself, not a stress test. Figure~\ref{fig:cross-order} reports two regimes on HODDI: \emph{Fixed Training} on orders $\{2, 3\}$, and \emph{Progressive Training} on $\{2, \ldots, k{-}1\}$ for each test order $k$. Hyperedge-centric baselines, whose order-$k$ predictions can only be supervised through actual order-$k$ training edges, collapse under Fixed and trail HGPM by up to $22$\,pp at order $6{+}$; under Progressive the gap narrows to $11$\,pp. HGPM's curve is essentially flat across both regimes, because the inclusion DAG routes every high-order target through a chain of inclusion edges to lower-order subsets that are dense in training, transporting the dense low-order signal into the sparse tail.

\begin{figure}[!t]
\centering
\includegraphics[width=\textwidth]{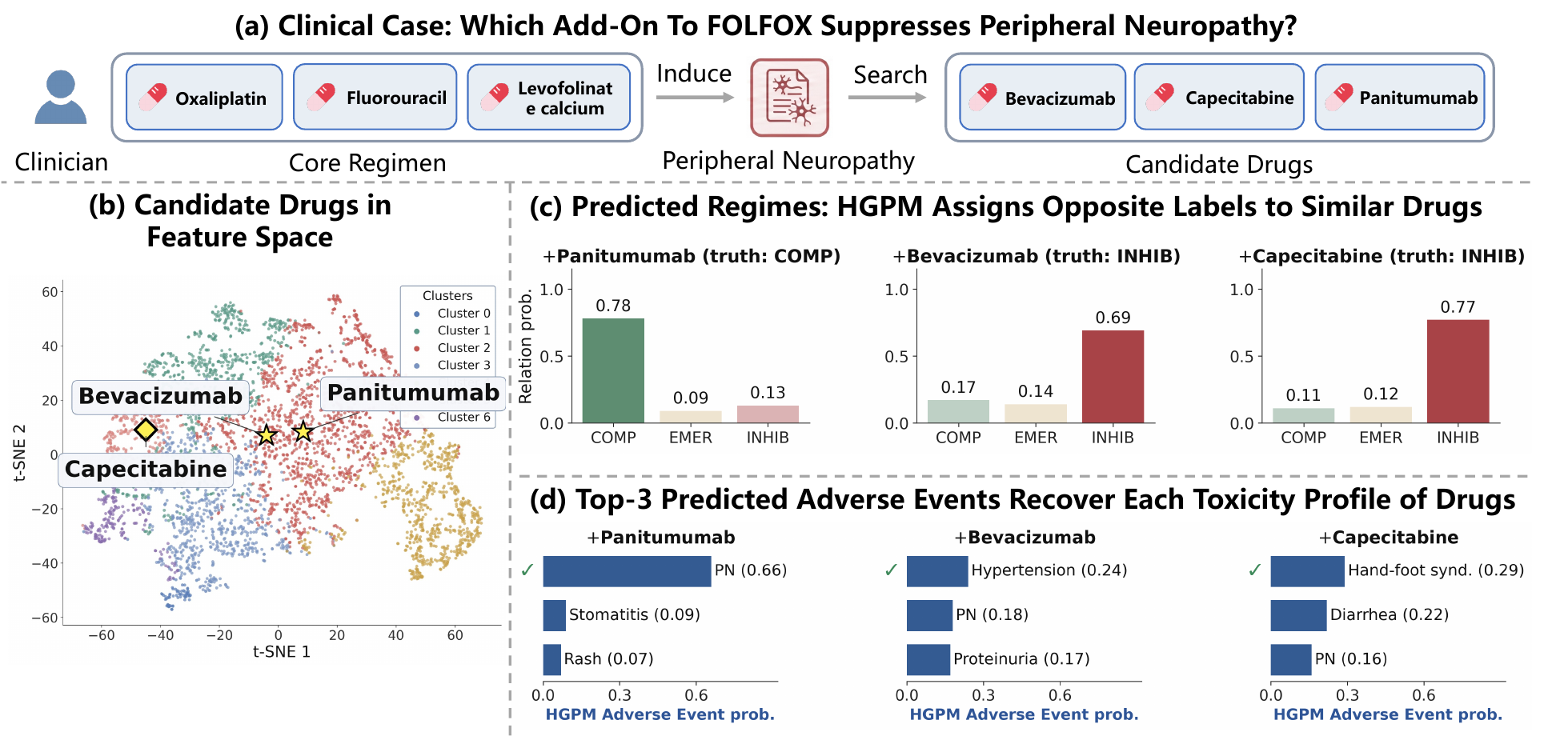}
\caption{\textbf{Case study} on selecting a suppressive add-on for peripheral neuropathy.}
\label{fig:case}
\vspace{-10pt}
\end{figure}

\subsection{Case Study}
\label{sec:exp-drugs-case}

We close with a polypharmacy case study. A JADER~\citep{pmda_jader} peripheral-neuropathy report documents a patient on FOLFOX~\citep{degramont2000folfox} (oxaliplatin, fluorouracil, levofolinate calcium), whose neurotoxicity is dominated by oxaliplatin~\citep{cavaletti2010cipn}. Adding a fourth drug requires distinguishing additions that preserve the side effect from those that suppress it (Figure~\ref{fig:case}a). Three candidates co-occur with this core in JADER: \emph{panitumumab}, \emph{bevacizumab}, and \emph{capecitabine}.
Panitumumab and bevacizumab look near-identical to a feature-similarity encoder: both are monoclonal antibodies in the same drug-feature cluster (Figure~\ref{fig:case}b). Yet the data records opposite outcomes: FOLFOX\,+\,panitumumab manifests peripheral neuropathy (label $1$, COMP), while FOLFOX\,+\,bevacizumab and FOLFOX\,+\,capecitabine suppress it (label $0$, INHIB). A feature-similarity predictor cannot resolve this; the contrast is readable only from the joint observation pattern of the 3-drug subset and its 4-drug superset.

For clinical decision support, we train two heads on top of HGPM's pretrained encoder: a relation head over $\{\text{COMP}, \text{EMER}, \text{INHIB}\}$ that types the candidate-regimen interaction, and an adverse-event head over the JADER side-effect vocabulary that ranks the resulting toxicities. The relation head separates the three branches correctly (Figure~\ref{fig:case}c): COMP$=0.78$ for panitumumab, INHIB$=0.69$ for bevacizumab, INHIB$=0.77$ for capecitabine. Bevacizumab is feature-similar to panitumumab yet flagged as the inhibitor, read from indicator-pair structure rather than drug embeddings, while the feature-similarity baseline collapses both inhibitors onto the compositional class (Figure \ref{fig:case-link-score}). Going beyond the binary class, the adverse-event head (Figure~\ref{fig:case}d) recovers each candidate's drug-class signature without any clinical labels: peripheral neuropathy for panitumumab, hypertension for the anti-VEGF bevacizumab, and hand-foot syndrome for the fluoropyrimidine capecitabine. The baseline collapses all three onto a shared neuropathy prior (Figure \ref{fig:case-baseline-ae}). 

\section{Conclusion}
\label{sec:conclusion}

We introduced HGPM, a hypergraph learning framework that shifts supervision from hyperedge existence to \emph{interaction compositionality}: the joint pattern of subset presence and absence that distinguishes higher-order relations from pairwise ones. HGPM tokenizes subsets as nodes of an inclusion DAG with composition-labeled edges and trains an inclusion-aware Transformer under a masked reconstruction objective. On eight hypergraph benchmarks and two drug-interaction corpora, HGPM matches or exceeds state-of-the-art methods. Two limitations remain: (1) The inclusion DAG grows combinatorially with order, making intractably computation; it is possible to sample the sub-DAG to mitigate the issue. (2) HGPM is per-target: each prediction builds and processes its own DAG, so whole-graph inference on large graphs requires per-node sampling. However, since the learned component is a standard Transformer, existing Transformer inference frameworks apply directly, keeping the practical cost manageable. Full discussion is in Appendix~\ref{sec:impact}.

{
\small
\bibliographystyle{plainnat}
\bibliography{refs}
}

\newpage\appendix
\section{Impact Statement}
\label{sec:impact}

This work introduces the Hypergraph Pattern Machine, a framework that introduces \emph{interaction compositionality} as a supervision target for hypergraph learning, distinguishing whether a higher-order relation is compositional, emergent, or inhibitory with respect to its subsets.
By making the compositionality structure of hypergraph data an explicit learning object, our framework aims to provide more interpretable and clinically actionable signals for higher-order interaction prediction, particularly in polypharmacy.
As machine learning systems for higher-order relational data are increasingly applied to clinical, biological, and societal contexts, it is important to consider the broader implications of advances in this representational paradigm.

\paragraph{Potential Benefits.}
HGPM may benefit several scientific and clinical domains.
In polypharmacy and pharmacology~\citep{wang2025hoddi}, the explicit \textsc{Comp}/\textsc{Emer}/\textsc{Inhib} regimes provide structurally distinct decision-support signals: compositional triples can in principle be simplified, emergent triples flag combinations whose members are jointly required, and inhibitory triples identify drugs that disrupt existing interactions---each pointing to a different prescription pattern.
In drug discovery, supervision on inclusion-hierarchy compositionality may surface higher-order functional patterns invisible to message-passing models, complementing existing pairwise drug--drug interaction pipelines.
Beyond the drug domain, HGPM applies to any higher-order relational data---recommendation, social-network analysis, and biological interaction modeling---where compositionality between observed and absent subsets carries scientifically meaningful signal.
More broadly, our framework contributes a shift in hypergraph learning from message passing on observed structure toward representations that encode the joint pattern of presence and absence, providing a principled interface between hypergraph machine learning and the structural questions practitioners actually care about.

\paragraph{Risks and Misuse Considerations.}
As with many advances in clinically adjacent machine learning, improved higher-order interaction prediction may lower barriers to misuse.
Mispredictions in clinical settings could affect patient care if HGPM outputs are used without physician oversight; training data such as HODDI~\citep{wang2025hoddi} is known to underreport rare adverse events, and HGPM may inherit this bias.
The same model that flags hazardous polypharmacy combinations could in principle be inverted to identify adversarial combinations with intentionally harmful effects, although our framework does not introduce new generative objectives beyond existing interaction-prediction pipelines.
We therefore emphasize that responsible deployment should follow established practices in clinical decision support, regulatory compliance, and dual-use risk assessment, and that any clinical use should retain human oversight at every step of the prescription pipeline.

\paragraph{Limitations and Scope.}
HGPM focuses on representation and supervision for hypergraph data with binary observation indicators, and does not directly address weighted, signed, or multi-relational hypergraphs.
Its computational cost grows with the size of the higher-order combinations being scored; the bounded per-order sampling keeps the inclusion DAG tractable but introduces approximation for very large combinations (e.g., $10+$-drug regimens).
Empirical results depend on the quality and coverage of training data, and validation on a fixed set of benchmarks does not substitute for experimental verification or domain expertise.
Finally, the compositionality structure HGPM supervises on is well-defined only for static observation indicators, and does not capture temporal dynamics, dose dependence, or patient-specific factors that affect real polypharmacy outcomes.

\paragraph{Future Directions.}
Future work may extend HGPM to weighted, signed, or multi-relational hypergraphs, where the binary observation indicator generalizes to richer label spaces.
Scaling pretraining to larger higher-order interaction corpora may improve cross-domain generalization, and integrating compositionality-aware representations with pharmacokinetic or mechanistic simulators could bridge data-driven prediction and mechanistic understanding.
From a societal perspective, developing evaluation protocols, dataset-governance practices, and clinical-deployment safeguards for higher-order interaction models will be important to ensure that the representational advances introduced here translate into responsible scientific and clinical use.

\section{Theoretical Analysis}
\label{sec:theory}

This appendix collects the formal results that ground HGPM's design. Section~\ref{sec:thm-coboundary} shows that the composition labels $\textsc{Comp}/\textsc{Emer}/\textsc{Inhib}$ of Eq.~\eqref{eq:comp-label} are exactly the values of the discrete forward-difference (cellular coboundary) of the hyperedge indicator on the inclusion lattice, with a corollary characterizing simplicial complexes as the EMER-free stratum. Section~\ref{sec:thm-mp-separation} establishes that hyperedge-MP encoders cannot represent the compositionality structure $\mathcal{C}(\mathcal{H})$: there exists a hypergraph pair on which all such encoders collapse to identical features but $\mathcal{C}(\mathcal{H})$ differs. Section~\ref{sec:thm-hgpm-separation} then exhibits a non-isomorphic, WL-equivalent hypergraph pair on which HGPM's centered representation distinguishes both, generically over weights.

\paragraph{Notational Conventions Used Throughout.}
$\mathcal{H} = (\mathcal{V}, \mathcal{E}, \mathbf{X})$ is a hypergraph with vertex set $\mathcal{V}$, hyperedge set $\mathcal{E} \subseteq 2^{\mathcal{V}}$, and node features $\mathbf{X}$. The hyperedge indicator is $\mathds{1}_S = \mathds{1}[S \in \mathcal{E}]$, with $\mathds{1}_\emptyset = 0$ by convention ($\emptyset$ is never a hyperedge). The compositionality structure $\mathcal{C}(\mathcal{H})$ collects covering pairs $(S, S')$ with $S \subset S'$, $|S'| = |S|+1$, and not-both-absent indicators, labeled by Eq.~\eqref{eq:comp-label}. The set of all covering pairs is denoted
\[
   \mathrm{Cov}(\mathcal{V}) = \{(S, S') : S \subset S' \subseteq \mathcal{V},\ |S'| = |S| + 1\},
\]
and we write $\mathrm{Cov}_+(\mathcal{V}) = \{(S, S') \in \mathrm{Cov}(\mathcal{V}) : S \neq \emptyset\}$ for covers with non-empty bottom face. We use ``covering pair'' (cover) and ``adjacent-order pair'' interchangeably throughout this appendix; the latter is the term used in the main text.

\subsection{Mathematical Identity of Compositionality}
\label{sec:thm-coboundary}

The forward-difference (cellular coboundary) operator $\delta$ acts on functions $f : 2^{\mathcal{V}} \to \mathbb{R}$ as
\[
   (\delta f)(S, S') = f(S') - f(S), \qquad (S, S') \in \mathrm{Cov}(\mathcal{V}).
\]
Applied to the hyperedge indicator viewed as a function $\mathds{1}_{(\cdot)} : S \mapsto \mathds{1}_S$, this yields a $\{-1, 0, +1\}$-valued $1$-cochain on the Hasse cover.

\begin{theorem}[Coboundary Identity]
\label{thm:comp-coboundary}
For every $(S, S') \in \mathrm{Cov}(\mathcal{V})$:

\smallskip
\textnormal{(a) (\emph{Local bijection}.)} $(\delta \mathds{1}_{(\cdot)})(S, S') \in \{-1, 0, +1\}$, and
\[
\begin{array}{rcl}
\mathrm{comp}(S, S') = \textsc{Emer}  &\iff& (\delta \mathds{1}_{(\cdot)})(S, S') = +1, \\
\mathrm{comp}(S, S') = \textsc{Inhib} &\iff& (\delta \mathds{1}_{(\cdot)})(S, S') = -1, \\
\mathrm{comp}(S, S') = \textsc{Comp}  &\iff& (\delta \mathds{1}_{(\cdot)})(S, S') = 0\ \text{and}\ \mathds{1}_S = 1,
\end{array}
\]
with both-absent pairs corresponding to $(\delta \mathds{1}_{(\cdot)})(S, S') = 0$ and $\mathds{1}_S = 0$.

\smallskip
\textnormal{(b) (\emph{Information equivalence}.)} Given $\mathcal{V}$ and the anchor $\mathds{1}_\emptyset = 0$, the indicator $\mathds{1}_{(\cdot)}$ is uniquely recoverable from $\mathcal{C}(\mathcal{H})$ by chain-induction along any saturated chain in $(2^{\mathcal{V}}, \subseteq)$.
\end{theorem}

\begin{proof}
\textbf{(a)} Since $\mathds{1}_S \in \{0, 1\}$, $(\delta \mathds{1}_{(\cdot)})(S, S') = \mathds{1}_{S'} - \mathds{1}_S \in \{-1, 0, +1\}$. The four ordered pairs $(\mathds{1}_S, \mathds{1}_{S'}) \in \{0, 1\}^2$ map bijectively to $((\delta \mathds{1}_{(\cdot)})(S, S'),\, \mathds{1}_S)$ as: $(0,1) \mapsto (+1, 0)$, $(1,0) \mapsto (-1, 1)$, $(1,1) \mapsto (0, 1)$, $(0,0) \mapsto (0, 0)$. The biconditionals follow from Eq.~\eqref{eq:comp-label}.

\textbf{(b)} For any $S \subseteq \mathcal{V}$, choose a saturated chain $\emptyset = S_0 \subset S_1 \subset \cdots \subset S_n = S$. Telescoping with $\mathds{1}_\emptyset = 0$,
\[
   \mathds{1}_S = \sum_{i=1}^n (\delta \mathds{1}_{(\cdot)})(S_{i-1}, S_i).
\]
By Clause~(a), each summand is determined by $\mathcal{C}(\mathcal{H})$: if the cover $(S_{i-1}, S_i)$ is informative, its label gives the value; if both indicators are zero, the value is zero (and $\mathds{1}_{S_{i-1}} = 0$ holds inductively, since otherwise the cover would be either COMP or INHIB and hence labeled in $\mathcal{C}(\mathcal{H})$). The right-hand sum is path-independent because $\delta \mathds{1}_{(\cdot)}$ is exact: it is the coboundary of a $0$-cochain.
\end{proof}

\begin{corollary}[Simplicial complexes are exactly the no-\textsc{Emer} stratum above the empty face]
\label{cor:scn-no-emer}
$\mathcal{E}$ is downward-closed among non-empty subsets --- $\sigma \in \mathcal{E}$ and $\emptyset \neq \tau \subset \sigma$ imply $\tau \in \mathcal{E}$, equivalently $\mathcal{H}$ is an abstract simplicial complex --- if and only if $(\delta \mathds{1}_{(\cdot)})(S, S') \neq +1$ for every $(S, S') \in \mathrm{Cov}_+(\mathcal{V})$. Equivalently, $\mathcal{C}(\mathcal{H})$ contains no \textsc{Emer} label on covers with $|S| \geq 1$.
\end{corollary}

\begin{proof}
($\Rightarrow$) Let $\mathcal{E}$ be downward-closed on non-empty subsets and fix $(S, S') \in \mathrm{Cov}_+(\mathcal{V})$. If $(\delta \mathds{1}_{(\cdot)})(S, S') = +1$, then $\mathds{1}_{S'} = 1, \mathds{1}_S = 0$. Since $S \neq \emptyset$ and $S \subset S'$, downward closure forces $S \in \mathcal{E}$, contradicting $\mathds{1}_S = 0$.

($\Leftarrow$) Suppose $(\delta \mathds{1}_{(\cdot)}) \neq +1$ on $\mathrm{Cov}_+(\mathcal{V})$. Fix $\sigma \in \mathcal{E}$ and non-empty $\tau \subset \sigma$; set $k = |\sigma| - |\tau| \geq 1$. Choose a saturated chain $\tau = T_0 \subset T_1 \subset \cdots \subset T_k = \sigma$. Each $T_{i-1} \supseteq \tau \neq \emptyset$, so $(T_{i-1}, T_i) \in \mathrm{Cov}_+(\mathcal{V})$. Telescoping,
\[
   \mathds{1}_\sigma - \mathds{1}_\tau = \sum_{i=1}^k (\delta \mathds{1}_{(\cdot)})(T_{i-1}, T_i) \leq 0,
\]
since each summand is in $\{-1, 0\}$ by hypothesis. Hence $\mathds{1}_\tau \geq \mathds{1}_\sigma = 1$, so $\mathds{1}_\tau = 1$ (as $\mathds{1}_{(\cdot)} \in \{0, 1\}$), i.e., $\tau \in \mathcal{E}$.
\end{proof}

\begin{remark}[Why the empty bottom face is excluded]
\label{rem:empty-bottom}
The restriction to $\mathrm{Cov}_+(\mathcal{V})$ is forced by the convention $\mathds{1}_\emptyset = 0$. For any $\mathcal{E} \neq \emptyset$ and any $\{v\} \in \mathcal{E}$, the cover $(\emptyset, \{v\})$ is automatically EMER ($\mathds{1}_\emptyset = 0, \mathds{1}_{\{v\}} = 1$); this bottom-layer EMER carries no information about closure structure. Concretely, $\mathcal{V} = \{1, 2, 3\}$ with $\mathcal{E} = \{\{1\}, \{2\}, \{3\}, \{1,2\}, \{1,3\}, \{2,3\}, \{1,2,3\}\}$ is downward-closed among non-empty subsets, yet $(\emptyset, \{i\})$ is EMER for each $i$. On $\mathrm{Cov}_+(\mathcal{V})$, all nine remaining covers have $(\delta \mathds{1}_{(\cdot)}) = 0$, consistent with Corollary~\ref{cor:scn-no-emer}.
\end{remark}

\begin{remark}[Implication for the simplicial-complex contrast]
\label{rem:scn-restriction}
Corollary~\ref{cor:scn-no-emer} formalizes the introduction's claim that simplicial-complex networks foreclose emergent and inhibitory transitions by construction. Any abstract simplicial complex satisfies
\[
   \bigl\{\mathds{1}_\emptyset = 0,\ (\delta \mathds{1}_{(\cdot)})(S, S') \neq +1\ \forall (S, S') \in \mathrm{Cov}_+(\mathcal{V})\bigr\},
\]
the EMER-free stratum on $\mathrm{Cov}_+(\mathcal{V})$. Simplicial-complex encoders therefore cannot, by virtue of their input domain, supervise on or distinguish EMER regimes on covers with non-empty bottom face --- i.e., on any genuine ``two existing faces of consecutive size'' transition. HGPM operates on the full $\{0, 1\}^{2^{\mathcal{V}}}$ and reads $\delta \mathds{1}_{(\cdot)}$ on $\mathrm{Cov}_+(\mathcal{V})$ via the composition bias $b^{\mathrm{comp}}_{ij}$ in attention.
\end{remark}

\begin{remark}[Connection to M\"obius inversion and Harsanyi dividends]
\label{rem:mobius}
The M\"obius transform of $\mathds{1}_{(\cdot)}$ on $(2^{\mathcal{V}}, \subseteq)$ is the unique $g : 2^{\mathcal{V}} \to \mathbb{Z}$ satisfying $\mathds{1}_S = \sum_{T \subseteq S} g(T)$, equivalently $g(S) = \sum_{T \subseteq S} (-1)^{|S| - |T|} \mathds{1}_T$ \citep{rota1964,harsanyi1963}. By Theorem~\ref{thm:comp-coboundary}(b), $\mathds{1}_{(\cdot)}$ --- and therefore $g$ --- is recoverable from $\mathcal{C}(\mathcal{H})$ given $\mathcal{V}$ and the anchor; this places $\mathcal{C}(\mathcal{H})$ inside the classical M\"obius / Harsanyi framework on Boolean lattices. We use the connection conceptually only; M\"obius coefficients are not computed by HGPM.

A quantitative caveat: ``$g(S) \neq 0 \Leftrightarrow$ EMER/INHIB pair on a chain in $2^S$'' fails by cancellation. Take $\mathcal{V} = \{a, b, c\}$, $\mathcal{E} = \{\{a\}, \{a, b\}\}$. Direct computation of all eight terms yields
\[
g(\mathcal{V}) = 0 - 1 - 0 - 0 + 1 + 0 + 0 - 0 = 0,
\]
yet $2^{\mathcal{V}}$ contains EMER covers $(\emptyset, \{a\})$ and $(\{b\}, \{a, b\})$, and INHIB covers $(\{a\}, \{a, c\})$ and $(\{a, b\}, \{a, b, c\})$. Thus $g(S) \neq 0$ implies non-constant $\mathds{1}_{(\cdot)}$ on $2^S$ (one direction); the converse can fail.
\end{remark}

\subsection{Hyperedge Message Passing Cannot Represent Compositionality}
\label{sec:thm-mp-separation}

\begin{definition}[Hyperedge-MP encoder]
\label{def:mp-encoder}
A \emph{hyperedge-MP encoder} of depth $T$ is any architecture of the form
\begin{align*}
\mathbf{m}^{(t)}_e &= \phi^{(t)}_E\bigl(\mathbf{h}^{(t)}_e,\ \{\!\!\{\mathbf{h}^{(t)}_v : v \in e\}\!\!\}\bigr), \\
\mathbf{h}^{(t+1)}_v &= \phi^{(t)}_V\bigl(\mathbf{h}^{(t)}_v,\ \{\!\!\{\mathbf{m}^{(t)}_e : v \in e\}\!\!\}\bigr), \\
\mathbf{h}^{(t+1)}_e &= \psi^{(t)}\bigl(\mathbf{m}^{(t)}_e,\ \{\!\!\{\mathbf{h}^{(t+1)}_v : v \in e\}\!\!\}\bigr),
\end{align*}
with $\mathbf{h}^{(0)}_v = \mathbf{x}_v$ and $\mathbf{h}^{(0)}_e$ a fixed initialization token, where each $\phi^{(t)}_E, \phi^{(t)}_V, \psi^{(t)}$ is a deterministic function of (own feature, multiset of incident features). HGNN \citep{feng2019hgnn}, AllDeepSets / AllSetTransformer \citep{chien2022allset}, ED-HNN \citep{wang2023edhnn}, and UniGCN \citep{huang2021unignn} all admit this factorization.
\end{definition}

\begin{lemma}[Bipartite-WL upper bound on hyperedge-MP]
\label{lem:wl-ub}
Let $B(\mathcal{H}) = (\mathcal{V} \sqcup \mathcal{E},\, I)$, $I = \{(v, e) : v \in e\}$, denote the incidence bipartite graph of $\mathcal{H}$. Bipartite-WL refines colors by
\[
c^{(t+1)}_v = \mathrm{HASH}\bigl(c^{(t)}_v,\ \{\!\!\{c^{(t)}_e : v \in e\}\!\!\}\bigr), \quad
c^{(t+1)}_e = \mathrm{HASH}\bigl(c^{(t)}_e,\ \{\!\!\{c^{(t)}_v : v \in e\}\!\!\}\bigr).
\]
For any hyperedge-MP encoder $\Phi$ of depth $T$ and any two hypergraphs $\mathcal{H}_1, \mathcal{H}_2$ on the same vertex set $\mathcal{V}$, if bipartite-WL produces, at every iteration $t \leq T$, pointwise-equal vertex colors $c^{(t)}_v[\mathcal{H}_1] = c^{(t)}_v[\mathcal{H}_2]$ for all $v \in \mathcal{V}$ together with pointwise-equal per-vertex multisets of incident hyperedge colors, then
\[
\mathbf{h}^{(T)}_v[\mathcal{H}_1] = \mathbf{h}^{(T)}_v[\mathcal{H}_2]\ \forall v \in \mathcal{V}, \qquad
\{\!\!\{\mathbf{h}^{(T)}_e[\mathcal{H}_1]\}\!\!\} = \{\!\!\{\mathbf{h}^{(T)}_e[\mathcal{H}_2]\}\!\!\}.
\]
\end{lemma}

\begin{proof}
Induction on $t$. At $t = 0$, $\mathbf{h}^{(0)}_v = \mathbf{x}_v$ depends only on $c^{(0)}_v = \mathbf{x}_v$; $\mathbf{h}^{(0)}_e$ is a fixed token, identical across $\mathcal{H}_1, \mathcal{H}_2$. Inductive step: assume at iteration $t$ both vertex features and per-vertex incident-hyperedge feature multisets agree across $\mathcal{H}_1, \mathcal{H}_2$. Then both inputs to $\phi^{(t)}_V$ at every $v$ are identical, so $\mathbf{h}^{(t+1)}_v[\mathcal{H}_1] = \mathbf{h}^{(t+1)}_v[\mathcal{H}_2]$. The hyperedge update is symmetric. This is the standard GIN-WL bound \citep{xu2019gin} lifted to bipartite incidence; see also \citet[Prop.\,3.3]{chien2022allset} and \citet[\S 3.2]{wang2023edhnn}.
\end{proof}

\begin{theorem}[Message-Passing Blindness]
\label{thm:mp-separation}
There exist hypergraphs $\mathcal{H}_1 = (\mathcal{V}, \mathcal{E}_1, \mathbf{X})$ and $\mathcal{H}_2 = (\mathcal{V}, \mathcal{E}_2, \mathbf{X})$ over the same vertex set $\mathcal{V}$ and constant node features $\mathbf{x}_v = \mathbf{x}_0$ such that:
\begin{enumerate}
\item[(a)] For every hyperedge-MP encoder $\Phi$ of any finite depth $T$,
\[
\mathbf{h}^{(T)}_v[\mathcal{H}_1] = \mathbf{h}^{(T)}_v[\mathcal{H}_2]\ \forall v \in \mathcal{V}, \quad \{\!\!\{\mathbf{h}^{(T)}_e[\mathcal{H}_1]\}\!\!\} = \{\!\!\{\mathbf{h}^{(T)}_e[\mathcal{H}_2]\}\!\!\}.
\]
\item[(b)] $\mathcal{C}(\mathcal{H}_1) \neq \mathcal{C}(\mathcal{H}_2)$ as labeled sets of adjacent-order pairs.
\end{enumerate}
Consequently, every map $f$ whose only $\mathcal{H}$-dependent input is the encoder output multisets satisfies $f(\mathcal{H}_1) = f(\mathcal{H}_2)$, hence cannot recover $\mathcal{C}(\mathcal{H})$ on at least one of $\mathcal{H}_1, \mathcal{H}_2$.
\end{theorem}

\begin{proof}
\textbf{Construction.} $\mathcal{V} = \{1, 2, 3, 4, 5, 6\}$, $\mathbf{x}_v = \mathbf{x}_0$ constant, $\mathcal{E}_1 = \{\{1, 2, 3\}, \{4, 5, 6\}\}$, $\mathcal{E}_2 = \{\{1, 2, 4\}, \{3, 5, 6\}\}$.

\textbf{Part (a) --- bipartite-WL collapse.} Both $B(\mathcal{H}_1)$ and $B(\mathcal{H}_2)$ are disjoint unions of two stars $K_{1, 3}$. Hand-running WL with constant initialization:
\begin{itemize}\itemsep0pt
\item $t = 0$: $c^{(0)}_v = A$ for all $v$; $c^{(0)}_e = \bot$ for both hyperedges.
\item $t = 1$: every vertex is in exactly one hyperedge, so $c^{(1)}_v = \mathrm{HASH}(A, \{\!\!\{\bot\}\!\!\}) = A_1$ uniformly; every hyperedge sees $\{\!\!\{A, A, A\}\!\!\}$, so $c^{(1)}_e = \mathrm{HASH}(\bot, \{\!\!\{A, A, A\}\!\!\}) = B_1$ uniformly.
\item $t = 2$: $c^{(2)}_v = A_2$, $c^{(2)}_e = B_2$ uniformly. Stable.
\end{itemize}
Color multisets are pointwise identical on $\mathcal{H}_1, \mathcal{H}_2$, and the per-vertex multiset of incident hyperedge colors equals $\{\!\!\{B_t\}\!\!\}$ for both. Lemma~\ref{lem:wl-ub} yields equal encoder outputs.

\textbf{Part (b) --- compositionality differs.} Take $(S, S') = (\{1, 2\}, \{1, 2, 3\})$. In $\mathcal{H}_1$: $\mathds{1}_{\{1, 2\}} = 0$ and $\mathds{1}_{\{1, 2, 3\}} = 1$, so $\mathrm{comp}(S, S') = \textsc{Emer}$ and $((S, S'), \textsc{Emer}) \in \mathcal{C}(\mathcal{H}_1)$. In $\mathcal{H}_2$: $\{1, 2, 3\} \notin \mathcal{E}_2$, so the cover is both-absent and excluded from $\mathcal{C}(\mathcal{H}_2)$. Hence $\mathcal{C}(\mathcal{H}_1) \neq \mathcal{C}(\mathcal{H}_2)$.

\textbf{Consequence.} Any $f$ depending on $\mathcal{H}$ only through encoder outputs gives $f(\mathcal{H}_1) = f(\mathcal{H}_2)$ by (a). Since $\mathcal{C}$ differs at the named pair $(S, S')$, $f$ errs on at least one of $\mathcal{H}_1, \mathcal{H}_2$.
\end{proof}

\begin{corollary}[Information-theoretic lower bound on MP-only comp prediction]
\label{cor:fano}
Let $\mathcal{D}$ be the uniform distribution on $\{\mathcal{H}_1, \mathcal{H}_2\}$ from Theorem~\ref{thm:mp-separation}, and consider the binary task $y(\mathcal{H}) = \mathds{1}[((\{1, 2\}, \{1, 2, 3\}), \textsc{Emer}) \in \mathcal{C}(\mathcal{H})]$. For any predictor $f$ that depends on $\mathcal{H}$ only through hyperedge-MP encoder output, $\mathrm{Pr}_{\mathcal{D}}[f(\mathcal{H}) = y(\mathcal{H})] = 1/2$, equivalently $I(\mathcal{C}(\mathcal{H});\, \Phi(\mathcal{H})) = 0 < H(\mathcal{C}(\mathcal{H})) = 1$ bit.
\end{corollary}

\begin{proof}
By Theorem~\ref{thm:mp-separation}(a), $\Phi(\mathcal{H}_1) = \Phi(\mathcal{H}_2)$ as multisets, so any $f$ depending only on $\Phi$ outputs the same value on both. The truth $y(\mathcal{H}_1) = 1, y(\mathcal{H}_2) = 0$ takes opposite values, giving accuracy exactly $1/2$.
\end{proof}

\begin{remark}[Representation expressiveness vs.\ end-to-end predictability]
\label{rem:repr-vs-pred}
Theorem~\ref{thm:mp-separation} is a representation-level claim: any predictor confined to encoder outputs errs on at least one of $\mathcal{H}_1, \mathcal{H}_2$. It does not preclude an end-to-end pipeline in which a head reads the raw incidence list $\mathcal{E}$ and computes $\mathds{1}_{S \in \mathcal{E}}$ by membership lookup, bypassing the encoder. Such a lookup-head trivially solves compositionality but renders the encoder dead weight. HGPM (Section~\ref{sec:method-tokenization}) lifts this barrier by exposing $\mathds{1}_S$ as a token attribute and the inclusion relation as an attention bias, so the encoder itself becomes the carrier of the compositionality signal.
\end{remark}

\begin{remark}[On the constant-features assumption]
\label{rem:constant-features}
The construction relies on $\mathbf{x}_v = \mathbf{x}_0$ for the WL collapse; with sufficiently distinguishing features, bipartite-WL separates $\mathcal{H}_1$ from $\mathcal{H}_2$ at iteration~1. The hypergraphs $\mathcal{H}_1, \mathcal{H}_2$ are also isomorphic as unlabeled hypergraphs (witness: the transposition $3 \leftrightarrow 4$). This is the standard regime for WL-style separation results in graph learning \citep{xu2019gin,morris2019kgnn}: a featureless minimal counterexample establishes non-universality, while CFI-style constructions \citep{cai1992optimal} extend the separation to non-isomorphic, non-constant-feature settings. Theorem~\ref{thm:hgpm-separation} below provides such a non-isomorphic witness pair on which Theorem~\ref{thm:mp-separation} continues to hold.
\end{remark}

\begin{remark}[$k$-WL barrier]
\label{rem:k-wl}
The pair $(\{1\}, \{3\})$ is co-incident in $\mathcal{H}_1$ (both lie in $\{1, 2, 3\}$) but not in $\mathcal{H}_2$ ($1 \in \{1, 2, 4\},\ 3 \in \{3, 5, 6\}$); a 2-WL refinement that tracks colors of node-pairs would distinguish $\mathcal{H}_1$ from $\mathcal{H}_2$ at iteration~1. None of the named hyperedge-MP encoders implements 2-WL. Lifting hyperedge-MP to 2-WL incurs $O(|\mathcal{V}|^2)$ memory; HGPM's per-target subset tokenization provides a fundamentally different scaling profile.
\end{remark}

\subsection{HGPM Strictly Separates a WL-Equivalent Pair}
\label{sec:thm-hgpm-separation}

We now exhibit a \emph{non-isomorphic} hypergraph pair on which all hyperedge-MP encoders collapse but HGPM's centered representation distinguishes the two, generically over weights. The non-isomorphism avoids the unlabeled-isomorphism degeneracy of Theorem~\ref{thm:mp-separation} and renders the separation robust to HGPM's permutation-equivariance.

\begin{construction}[Non-isomorphic WL-equivalent witness pair]
\label{constr:t3-pair}
Let $\mathcal{V}' = \{1, \ldots, 6\}$ with constant features $\mathbf{x}_v = \mathbf{x}_0$, and define
\[
\mathcal{E}_1' = \bigl\{\{1, 2, 3\},\ \{1, 2, 4\},\ \{3, 5, 6\},\ \{4, 5, 6\}\bigr\},
\]
\[
\mathcal{E}_2' = \bigl\{\{1, 2, 3\},\ \{1, 4, 5\},\ \{2, 4, 6\},\ \{3, 5, 6\}\bigr\}.
\]
Both hypergraphs are 3-uniform with 4 hyperedges; every $v \in \mathcal{V}'$ has degree 2. We refer to the resulting hypergraphs as $\mathcal{H}_1', \mathcal{H}_2'$.
\end{construction}

\paragraph{Verified Structural Properties (Exhaustive Enumeration).}
\begin{itemize}\itemsep0pt
\item Pair-intersection multisets: $\{|e \cap e'| : e, e' \in \mathcal{E}_1', e \neq e'\} = \{0, 0, 1, 1, 2, 2\}$ versus $\{|e \cap e'| : e, e' \in \mathcal{E}_2', e \neq e'\} = \{1, 1, 1, 1, 1, 1\}$. Pair-intersection histogram is an isomorphism invariant, hence $\mathcal{H}_1' \not\cong \mathcal{H}_2'$. (Brute-force check over all $720$ vertex permutations confirms no isomorphism exists.)
\item Bipartite-WL stable colorings on $\mathcal{H}_1'$ and $\mathcal{H}_2'$ are pointwise identical (Lemma~\ref{lem:wl-collapse-prime} below).
\end{itemize}

\begin{lemma}[Bipartite-WL collapse on Construction~\ref{constr:t3-pair}]
\label{lem:wl-collapse-prime}
On both $\mathcal{H}_1'$ and $\mathcal{H}_2'$, bipartite-WL stabilizes at iteration~2 with a single $\mathcal{V}'$-side color and a single $\mathcal{E}'$-side color. Consequently, by Lemma~\ref{lem:wl-ub}, every hyperedge-MP encoder $\Phi$ produces identical per-vertex features and identical hyperedge-feature multisets on $\mathcal{H}_1', \mathcal{H}_2'$.
\end{lemma}

\begin{proof}
At $t = 0$ all vertices have color $A$ (constant features) and all hyperedges color $B = \mathrm{HASH}(\mathrm{size} = 3, \{\!\!\{A, A, A\}\!\!\})$. At $t = 1$, every vertex is in exactly two hyperedges of color $B$ and updates to $A_1 = \mathrm{HASH}(A, \{\!\!\{B, B\}\!\!\})$; every hyperedge updates to $B_1 = \mathrm{HASH}(B, \{\!\!\{A, A, A\}\!\!\})$. Stable: $t \geq 2$ leaves all colors unchanged. Identical on $\mathcal{H}_1'$ and $\mathcal{H}_2'$.
\end{proof}

Let $V^* = \{1, 2, 5, 6\} \subset \mathcal{V}'$. The remaining two targets $\{3, 4\}$ have coincident centered token sets across $\mathcal{H}_1', \mathcal{H}_2'$ because the two observed hyperedges containing $3$ (resp.\ $4$) are common to $\mathcal{E}_1' \cap \mathcal{E}_2'$.

\begin{lemma}[Token cardinality differs at $V^*$]
\label{lem:hgpm-token-diff}
For target $c \in V^*$, let $\mathcal{S}_c[\mathcal{H}]$ denote the centered token set under deterministic sampling that retains the center singleton $\{c\}$, all observed hyperedges containing $c$, and all $2$-subsets of those observed hyperedges that contain $c$ (the latter are classified as $\tau = \mathrm{neg}$ tokens with $\mathds{1}_S = 0$). Then for every $c \in V^*$,
\[
   |\mathcal{S}_c[\mathcal{H}_1']| = 6, \qquad |\mathcal{S}_c[\mathcal{H}_2']| = 7,
\]
and the corresponding token attribute multisets $\{\!\!\{(|S|, \mathds{1}_S, \tau(S))\}\!\!\}$ differ.
\end{lemma}

\begin{proof}
The structural driver: the two observed hyperedges in $\mathcal{H}_1'$ that contain $c \in V^*$ form a $2$-overlap pair (sharing two vertices including $c$), whereas in $\mathcal{H}_2'$ they form a $1$-overlap pair (sharing only $c$).
\begin{itemize}\itemsep0pt
\item $c = 1$: $\mathcal{H}_1'$ obs $\{1,2,3\}, \{1,2,4\}$ (overlap $\{1,2\}$), $2$-subsets containing $1$: $\{1,2\}, \{1,3\}, \{1,4\}$ (three distinct). $\mathcal{H}_2'$ obs $\{1,2,3\}, \{1,4,5\}$ (overlap $\{1\}$), $2$-subsets: $\{1,2\}, \{1,3\}, \{1,4\}, \{1,5\}$ (four distinct).
\item $c = 2$: $\mathcal{H}_1'$ obs $\{1,2,3\}, \{1,2,4\}$, $2$-subs $\{1,2\}, \{2,3\}, \{2,4\}$. $\mathcal{H}_2'$ obs $\{1,2,3\}, \{2,4,6\}$, $2$-subs $\{1,2\}, \{2,3\}, \{2,4\}, \{2,6\}$.
\item $c = 5$: $\mathcal{H}_1'$ obs $\{3,5,6\}, \{4,5,6\}$, $2$-subs $\{3,5\}, \{5,6\}, \{4,5\}$. $\mathcal{H}_2'$ obs $\{1,4,5\}, \{3,5,6\}$, $2$-subs $\{1,5\}, \{4,5\}, \{3,5\}, \{5,6\}$.
\item $c = 6$: $\mathcal{H}_1'$ obs $\{3,5,6\}, \{4,5,6\}$, $2$-subs $\{3,6\}, \{5,6\}, \{4,6\}$. $\mathcal{H}_2'$ obs $\{2,4,6\}, \{3,5,6\}$, $2$-subs $\{2,6\}, \{4,6\}, \{3,6\}, \{5,6\}$.
\end{itemize}
In all four cases $|\mathcal{S}_c[\mathcal{H}_1']| = 1 + 2 + 3 = 6$ and $|\mathcal{S}_c[\mathcal{H}_2']| = 1 + 2 + 4 = 7$, with attribute multisets differing in the count of $(|S|, \mathds{1}_S, \tau(S)) = (2, 0, \mathrm{neg})$ tokens (3 versus 4).
\end{proof}

\begin{lemma}[Inclusion DAGs are non-isomorphic at $c \in V^*$]
\label{lem:dag-noniso}
For every $c \in V^*$, the inclusion DAGs $\mathcal{G}_c[\mathcal{H}_1'] = (\mathcal{S}_c[\mathcal{H}_1'], \mathcal{R}_c[\mathcal{H}_1'])$ and $\mathcal{G}_c[\mathcal{H}_2']$ are non-isomorphic as labeled DAGs.
\end{lemma}

\begin{proof}
We show the label-restricted out-degree multiset $\{\!\!\{\deg^+ S : S \in \mathcal{S}_c, \tau(S) = \mathrm{neg}\}\!\!\}$ differs across $\mathcal{H}_1', \mathcal{H}_2'$. The out-degree of a $\mathrm{neg}$ token equals the number of observed hyperedges containing $c$ that have it as a $2$-subset.

In $\mathcal{H}_1'$ at any $c \in V^*$: by the $2$-overlap pair structure, exactly one $\mathrm{neg}$ token is a $2$-subset of \emph{both} observed hyperedges (the shared $2$-overlap; e.g., $\{1, 2\}$ at $c = 1$); the other two $\mathrm{neg}$ tokens are subsets of one observed hyperedge each. Out-degree multiset: $\{2, 1, 1\}$.

In $\mathcal{H}_2'$ at any $c \in V^*$: by the $1$-overlap pair structure, every $\mathrm{neg}$ token is a $2$-subset of exactly one observed hyperedge. Out-degree multiset: $\{1, 1, 1, 1\}$.

The label-restricted out-degree multiset is invariant under any DAG isomorphism preserving $\tau$. Multisets $\{2, 1, 1\} \neq \{1, 1, 1, 1\}$ (differ in cardinality 3 vs 4 and maximum value 2 vs 1), hence $\mathcal{G}_c[\mathcal{H}_1'] \not\cong \mathcal{G}_c[\mathcal{H}_2']$ for every $c \in V^*$.
\end{proof}

\begin{lemma}[Generic separation on non-isomorphic labeled DAGs]
\label{lem:gen-sep}
Let $\Theta = \prod_i \mathbb{R}^{d_i}$ be HGPM's parameter space (attribute embeddings $\mathbf{e}^{\mathrm{ord}}, \mathbf{e}^{\mathrm{ex}}, \mathbf{e}^{\mathrm{src}}, \mathbf{e}^{\mathrm{view}}, \mathbf{e}^{\mathrm{pos}}$, structural-bias parameters $b^{\mathrm{dir}}, b^{\mathrm{ord}}, b^{\mathrm{ovlp}}, b^{\mathrm{sib}}, b^{\mathrm{comp}}$, per-layer attention $\{W_Q^{(\ell)}, W_K^{(\ell)}, W_V^{(\ell)}, W_O^{(\ell)}\}$, FFN/LayerNorm weights, attention-pool weights $(\mathbf{w}, \mathbf{W})$). For two finite labeled DAGs $\mathcal{G}_1, \mathcal{G}_2$, define $f_i(\theta) = \mathrm{HGPM}_c(\theta)(\mathcal{G}_i)$. Assume:
\begin{itemize}\itemsep0pt
\item[(A1)] $f_1, f_2$ are real-analytic on $\Theta$ (closed under finite compositions of softmax, $\tanh$, GELU, linear maps).
\item[(A2)] HGPM is permutation-equivariant in token positions; $f_i$ depends only on the labeled-DAG isomorphism class of $\mathcal{G}_i$.
\item[(A3)] There exists $\theta_0 \in \Theta$ with $f_1(\theta_0) \neq f_2(\theta_0)$.
\end{itemize}
Then the equality locus $\Theta_= = \{\theta : f_1(\theta) = f_2(\theta)\}$ is a real-analytic proper subvariety of $\Theta$, hence Lebesgue-null.
\end{lemma}

\begin{proof}
$g(\theta) = f_1(\theta) - f_2(\theta) \in \mathbb{R}^d$ is real-analytic by (A1). $\Theta_= = g^{-1}(0) = \bigcap_{j=1}^d g_j^{-1}(0)$. By (A3), at least one component $g_j$ is not identically zero. The zero set of a non-zero real-analytic function on a connected open subset of $\mathbb{R}^N$ has Lebesgue measure zero \citep{mityagin2020,krantzparks2002}. Hence $\Theta_= \subseteq g_j^{-1}(0)$ has measure zero.
\end{proof}

For the statement below we write $\mathbf{h}_c^{\mathrm{ctx}}[\theta, \mathcal{H}] := \sum_{S \in \mathcal{S}_c[\mathcal{H}]} \alpha_S \mathbf{h}_S^{(L)}$ for the attention-pooled context branch of the readout in Appendix~\ref{sec:downstream-readouts}, viewed as a function of HGPM weights $\theta$; this is the second concatenand of $\mathbf{h}_c^{\mathrm{node}}$, and a separation in $\mathbf{h}_c^{\mathrm{ctx}}$ lifts to a separation in $\mathbf{h}_c^{\mathrm{node}}$ since $\mathbf{h}_{\{c\}}^{(L)}$ is a fixed function of $\theta$ and the input subset features alone.

\begin{theorem}[HGPM Separation]
\label{thm:hgpm-separation}
Consider Construction~\ref{constr:t3-pair}.
\begin{enumerate}
\item[(a)] Every hyperedge-MP encoder produces identical per-vertex features and identical hyperedge-feature multisets on $\mathcal{H}_1', \mathcal{H}_2'$ (Lemma~\ref{lem:wl-collapse-prime}).
\item[(b)] For every $c \in V^* = \{1, 2, 5, 6\}$, the inclusion DAGs satisfy $\mathcal{G}_c[\mathcal{H}_1'] \not\cong \mathcal{G}_c[\mathcal{H}_2']$ (Lemma~\ref{lem:dag-noniso}). Consequently, for any HGPM parameterization $\theta$ outside a Lebesgue-null subvariety $\Theta_=^c \subset \Theta$, the centered representation $\mathbf{h}_c^{\mathrm{ctx}}[\theta, \mathcal{H}_1']$ differs from $\mathbf{h}_c^{\mathrm{ctx}}[\theta, \mathcal{H}_2']$.
\item[(c)] Compositionality structures differ: $\mathcal{C}(\mathcal{H}_1') \neq \mathcal{C}(\mathcal{H}_2')$. The labeled covering pair $((\{1, 2\}, \{1, 2, 4\}), \textsc{Emer})$ lies in $\mathcal{C}(\mathcal{H}_1') \setminus \mathcal{C}(\mathcal{H}_2')$.
\end{enumerate}
\end{theorem}

\begin{proof}
\textbf{(a)} By Lemma~\ref{lem:wl-collapse-prime}.

\textbf{(b)} Fix $c \in V^*$. Lemma~\ref{lem:dag-noniso} gives $\mathcal{G}_c[\mathcal{H}_1'] \not\cong \mathcal{G}_c[\mathcal{H}_2']$. We construct an explicit separating witness $\theta_0 \in \Theta$ to verify (A3) of Lemma~\ref{lem:gen-sep}.

\emph{Witness construction.} Choose:
\begin{itemize}\itemsep0pt
\item Embeddings $\mathbf{e}^{\mathrm{ord}} = \mathbf{e}^{\mathrm{ex}} = \mathbf{e}^{\mathrm{view}} = \mathbf{e}^{\mathrm{pos}} = 0$, and $\mathbf{e}^{\mathrm{src}}$ orthonormal across $\{\mathrm{center}, \mathrm{obs}, \mathrm{neg}\}$ in $\mathbb{R}^d$ with $d \geq 4$.
\item Structural biases zero: $b^{\mathrm{dir}} = b^{\mathrm{ord}} = b^{\mathrm{ovlp}} = b^{\mathrm{sib}} = b^{\mathrm{comp}} = 0$.
\item $W_Q^{(\ell)} = W_K^{(\ell)} = 0$ (uniform attention weights $1/|\mathcal{S}_c|$); $W_V^{(\ell)} = 0$ (zero attention contribution); $W_O^{(\ell)}$ arbitrary; FFN identity (zero bias, identity weights).
\item Attention-pool weights $\mathbf{w} = \mathbf{W} = 0$ (uniform mean pool $\alpha_S = 1/|\mathcal{S}_c|$).
\end{itemize}
Under this $\theta_0$ the residual stream is unchanged at every layer ($W_V = 0$ kills attention output, FFN is identity), so $\mathbf{h}_S^{(L)}$ is determined by $\tau(S)$ alone modulo a deterministic per-token LayerNorm. Define $v_\tau$ as the common feature of all type-$\tau$ tokens (well-defined because all type-$\tau$ inputs are equal under $\theta_0$). The pooled output is
\[
   \mathbf{h}_c^{\mathrm{ctx}}[\mathcal{H}] = \frac{1}{|\mathcal{S}_c[\mathcal{H}]|}\sum_{\tau \in \{\mathrm{center}, \mathrm{neg}, \mathrm{obs}\}} n_\tau[\mathcal{H}]\, v_\tau.
\]
By Lemma~\ref{lem:hgpm-token-diff}, for $c \in V^*$ the type counts are $(n_{\mathrm{center}}, n_{\mathrm{neg}}, n_{\mathrm{obs}}) = (1, 3, 2)$ in $\mathcal{H}_1'$ (cardinality $6$) and $(1, 4, 2)$ in $\mathcal{H}_2'$ (cardinality $7$). Subtracting,
\[
   \mathbf{h}_c^{\mathrm{ctx}}[\mathcal{H}_1'] - \mathbf{h}_c^{\mathrm{ctx}}[\mathcal{H}_2'] = \tfrac{1}{42}\bigl(v_{\mathrm{center}} - 3\, v_{\mathrm{neg}} + 2\, v_{\mathrm{obs}}\bigr),
\]
using $\tfrac{1}{6} - \tfrac{1}{7} = \tfrac{1}{42}$, $\tfrac{3}{6} - \tfrac{4}{7} = -\tfrac{3}{42}$, $\tfrac{2}{6} - \tfrac{2}{7} = \tfrac{2}{42}$. For $d \geq 4$, the LayerNorm-images of three orthonormal unit basis vectors remain linearly independent (their non-zero coordinates lie at distinct indices), so $\{v_{\mathrm{center}}, v_{\mathrm{neg}}, v_{\mathrm{obs}}\}$ are linearly independent. The coefficients $(1, -3, 2)$ are not all zero, hence $\mathbf{h}_c^{\mathrm{ctx}}[\mathcal{H}_1'] \neq \mathbf{h}_c^{\mathrm{ctx}}[\mathcal{H}_2']$, establishing (A3).

(A1) holds because softmax / $\tanh$ / GELU / linear maps are real-analytic and finite compositions thereof are real-analytic. (A2) holds because HGPM's structural biases depend only on labeled pairwise relations between tokens, not on absolute positions. By Lemma~\ref{lem:gen-sep}, the equality locus $\Theta_=^c$ is Lebesgue-null.

\textbf{(c)} By Theorem~\ref{thm:comp-coboundary}(b), $\mathcal{E}_1' \neq \mathcal{E}_2'$ implies $\mathcal{C}(\mathcal{H}_1') \neq \mathcal{C}(\mathcal{H}_2')$. Direct verification: $\{1, 2\} \notin \mathcal{E}_1'$ and $\{1, 2, 4\} \in \mathcal{E}_1'$ give EMER in $\mathcal{H}_1'$; $\{1, 2, 4\} \notin \mathcal{E}_2'$ excludes the cover from $\mathcal{C}(\mathcal{H}_2')$.
\end{proof}

\begin{remark}[One-way separation, not strict containment]
\label{rem:one-way}
Theorem~\ref{thm:hgpm-separation} establishes a one-way separation on the compositionality task: there is a hypergraph pair on which HGPM's representation distinguishes what hyperedge-MP cannot. It does \emph{not} establish strict containment of expressivity classes. HGPM is per-target with budget-bounded subset sampling and cannot natively simulate hyperedge-MP's cross-vertex global propagation. The two function classes are \emph{incomparable}; HGPM's design buys expressivity on inclusion-hierarchy patterns at the cost of unrestricted graph propagation.
\end{remark}

\begin{remark}[Targets $c \in \{3, 4\}$ are sampling-deficient on Construction~\ref{constr:t3-pair}]
\label{rem:c-deficient}
At $c \in \{3, 4\}$ the centered token sets $\mathcal{S}_c[\mathcal{H}_1']$ and $\mathcal{S}_c[\mathcal{H}_2']$ coincide as multisets, since the two observed hyperedges containing $c$ are identical across $\mathcal{H}_1', \mathcal{H}_2'$ (e.g., for $c = 3$ both contain $\{1, 2, 3\}$ and $\{3, 5, 6\}$). HGPM's deterministic sampling at these targets is therefore blind to the difference; this is a property of the construction rather than a defect of HGPM. For tasks aggregating across $c \in \mathcal{V}'$ (node classification, graph classification), the four targets in $V^*$ already yield an output difference, so the global task is solvable on this construction by HGPM but not by hyperedge-MP.
\end{remark}

\begin{remark}[Generic vs.\ trained weights]
\label{rem:generic-caveat}
``Generic'' in Theorem~\ref{thm:hgpm-separation}(b) means Lebesgue-a.e.\ in $\Theta$. Trained HGPM weights are not random; the practical claim is that no Zariski-open obstruction blocks the separation, and any absolutely continuous initializer followed by gradient descent avoids measure-zero loci with probability one --- a folklore observation rather than a theorem; we make no formal claim about trained $\theta$. If HGPM employs ReLU rather than GELU in the FFN, real-analyticity (A1) fails on a measure-zero locus where preactivations cross zero; the conclusion still holds for almost all $\theta$ by stratified analyticity, or by replacing ReLU with GELU.
\end{remark}

\section{Technical Details}
\label{sec:tech-details}

We collect the implementation specifics deferred from Section~\ref{sec:method}, organized along the pipeline: building the inclusion DAG (Section~\ref{sec:tech-dag}), the encoder and pretraining heads on top of it (Section~\ref{sec:tech-arch}), and the finetuning side, both transfer and downstream readouts (Section~\ref{sec:tech-finetune}).

\subsection{Inclusion DAG Construction}
\label{sec:tech-dag}
\label{sec:dag-construction}

\paragraph{Negative subsets.} For each observed $e \in \mathcal{E}$ with $c \in e$, three guarded operations propose negative candidates:
\begin{align*}
\textsc{Drop} &: \text{if } |e| > 2,\ \text{sample } v \in e \setminus \{c\},\ \text{emit } e \setminus \{v\}, \\
\textsc{Add} &: \text{if } |e| < K_{\max},\ \text{sample } u \in \mathcal{V} \setminus e,\ \text{emit } e \cup \{u\}, \\
\textsc{Swap} &: \text{repeat } r \text{ times: sample } (v, u) \in (e \setminus \{c\}) \times (\mathcal{V} \setminus e),\ \text{emit } (e \setminus \{v\}) \cup \{u\},
\end{align*}
with $r \in \{1, 2, 3\}$ tuned per dataset. A candidate $S$ is accepted iff $c \in S$ and $\mathds{1}_S = 0$. By construction the resulting pool $\mathcal{N}_c$ is lattice-adjacent to $\mathcal{E}$ (\textsc{Drop}/\textsc{Add} land one order away, \textsc{Swap} stays at the same order with symmetric difference two), so every negative participates in some adjacent-order pair with an observed counterpart, which is what activates the \textsc{Inhib} and \textsc{Emer} branches of Eq.~\eqref{eq:comp-label}.

\paragraph{Bounded sampling and stochastic views.} With per-order cap $K_o$ and negative quota $B^{\mathrm{neg}}_o \in \{1, 2\}$, each order $o \in \{2, \ldots, K_{\max}\}$ is filled observed-heavy: take $n_+ = \min(|P_o|, K_o)$ uniform draws (without replacement) from $P_o = \{e \in \mathcal{E}: c \in e, |e| = o\}$, then $n_- = \min(B^{\mathrm{neg}}_o, K_o - n_+)$ from $N_o = \{S \in \mathcal{N}_c: |S| = o\}$. The edge set is then recomputed as $\mathcal{R}_c = \{(S, S') \in \mathcal{S}_c^2 : S \subset S',\ |S'| = |S| + 1\}$. We sample $V = 2$ independent DAGs per center and concatenate them into a single token sequence of length $V \sum_o K_o$. Adjacency-keyed biases ($b^{\mathrm{dir}}$, $b^{\mathrm{sib}}$) are computed view-locally; content-keyed biases ($b^{\mathrm{comp}}$, $b^{\mathrm{ord}}$, $b^{\mathrm{ovlp}}$) fire across views, so the encoder averages compositionality cues over the $V$ samples of $\mathcal{N}_c$.

\subsection{Encoder Architecture and Pretraining Heads}
\label{sec:tech-arch}

\paragraph{Token embedding.} Extending the four-attribute lookup of Section~\ref{sec:method-encoder}, we add a fifth additive term, the view embedding $\mathbf{e}^{\mathrm{view}}_{v(S)}$, so identical subsets across views remain distinguishable. The feature projector is a 2-layer MLP (Linear, ReLU, Dropout, Linear) at width $d$; every attribute table reserves index $0$ for padding, and the position table has capacity $1024$.

\paragraph{Structural biases.}
\label{sec:aux-biases}
The composition bias $b^{\mathrm{comp}}_{ij}$ is a $7$-row per-head embedding table, zero-initialized, indexed by the asymmetric role pair $(\tau(S_i), \tau(S_j))$ with $i$ the query (Table~\ref{tab:exist-transition-bucket}). The asymmetry exploits that every non-center token contains $\{c\}$, so a center query reads informative role signals (buckets $5$, $6$); mirror cases collapse into bucket $0$. The auxiliary term $b^{\mathrm{aux}} = b^{\mathrm{ord}} + b^{\mathrm{ovlp}} + b^{\mathrm{sib}}$ is also per-head and zero-initialized: $b^{\mathrm{ord}}$ uses $7$ buckets on the signed gap $|S_i| - |S_j|$ at $\{\leq -3, -2, -1, 0, 1, 2, \geq 3\}$; $b^{\mathrm{ovlp}}$ uses $5$ buckets on the Jaccard ratio $|S_i \cap S_j|/|S_i \cup S_j|$ at thresholds $\{0, 0.25, 0.5, 0.75\}$; $b^{\mathrm{sib}}$ is a scalar gated by whether $S_i \neq S_j$ share an immediate parent or child in $\mathcal{G}_c$.

\begin{table}[h]
\caption{Asymmetric bucket index for $b^{\mathrm{comp}}_{ij}$. $i$ is the query, $j$ the key. Bucket $0$ is the catch-all for any role tuple outside the six explicit cases (including pairs whose key is the center).}
\label{tab:exist-transition-bucket}
\centering
\small
\begin{tabular}{lc}
\toprule
$(\tau(S_i),\, \tau(S_j))$ & Bucket \\
\midrule
$(\mathrm{obs},\, \mathrm{obs})$ & $1$ \\
$(\mathrm{obs},\, \mathrm{neg})$ & $2$ \\
$(\mathrm{neg},\, \mathrm{obs})$ & $3$ \\
$(\mathrm{neg},\, \mathrm{neg})$ & $4$ \\
$(\mathrm{center},\, \mathrm{obs})$ & $5$ \\
$(\mathrm{center},\, \mathrm{neg})$ & $6$ \\
all other role tuples & $0$ \\
\bottomrule
\end{tabular}
\end{table}

\paragraph{Pretraining heads.} The semantic head $\mathrm{Pred}: \mathbb{R}^d \to \mathbb{R}^d$ is a 2-layer MLP (Linear, GELU, Dropout, Linear) at width $d$, regressed against the EMA teacher $\mathrm{TeacherMLP}: \mathbb{R}^{d_{\mathrm{in}}} \to \mathbb{R}^d$ (Linear, GELU, Linear at width $d$) on masked positions $\mathcal{M}$. The existence head is a $d \to d/2 \to 1$ bottleneck with GELU and Dropout, emitting the BCE logit at every non-padding token. The mask token $\mathbf{e}_{\mathrm{mask}} \in \mathbb{R}^d$ is initialized elementwise from $\mathcal{N}(0, 0.02^2)$; the per-batch mask ratio is $\rho = 0.2$ over non-center tokens.

\subsection{Finetuning: Weight Transfer and Downstream Readouts}
\label{sec:tech-finetune}

\paragraph{Weight transfer.} At finetune, a non-strict state-dict load keeps the encoder backbone and structural-bias tables: \texttt{feature\_projection}, the five attribute embeddings (\texttt{order}, \texttt{exist}, \texttt{exist\_source}, \texttt{view}, \texttt{position}), the four bias tables (\texttt{edge\_direction}, \texttt{order\_distance}, \texttt{overlap\_bucket}, \texttt{exist\_transition}), the sibling-bias scalar, every Transformer block, and the final LayerNorm. The pretraining heads, $\mathbf{e}_{\mathrm{mask}}$, $\mathrm{TeacherMLP}$, and the readouts below are reinitialized; Appendix~\ref{sec:full-ablation} reports the pretrain-vs-scratch ablation.

\paragraph{Node classification.}
\label{sec:downstream-readouts}
An attention pool $\alpha_S = \mathrm{softmax}_S(\mathbf{w}^{\top} \tanh(\mathbf{W} \mathbf{h}_S^{(L)}))$ is concatenated with the center, $\mathbf{h}_c^{\mathrm{node}} = [\mathbf{h}_{\{c\}}^{(L)} \,\|\, \sum_S \alpha_S \mathbf{h}_S^{(L)}]$, and decoded by an MLP under cross-entropy.

\paragraph{Higher-order drug interaction.} For a combination $G = \{d_1, \ldots, d_m\}$ optionally conditioned on a side-effect $s$, we build one inclusion DAG per drug, average their center embeddings $\mathbf{h}_G = \tfrac{1}{m} \sum_j \mathbf{h}_{\{d_j\}}^{(L)}$, and decode $\hat{y}_{G,s} = \sigma(\mathrm{MLP}([\mathbf{h}_G \,\|\, \mathbf{e}_s]))$ with a learnable $\mathbf{e}_s$ (zero if unconditional) under binary cross-entropy.

\section{Datasets}
\label{sec:dataset-details}

\subsection{General Hypergraph Benchmarks}
\label{sec:dataset-general}

We adopt the standard $8$-benchmark suite for hypergraph node classification established by AllSet~\citep{chien2022allset} and extended with the four heterophilic datasets in ED-HNN~\citep{wang2023edhnn}. The suite splits into four homophilic citation / co-authorship datasets and four structurally diverse / heterophilic datasets. Across all eight datasets we use the same $50/25/25$ random train/validation/test split, and report mean $\pm$ std over ten random seeds. Table~\ref{tab:dataset-stats} summarizes per-dataset statistics; per-dataset descriptions follow.

\begin{table}[h]
\caption{Statistics of the eight hypergraph node classification benchmarks. avg.\,$|e|$ is the mean hyperedge size; CE-hom.\ is the clique-expansion homophily ratio (values $\approx 0.5$ indicate heterophilic behavior); feat.\,dim is the dimension of node input features (synthetic CSBM-style Gaussian features, dim $100$, are imputed for the heterophilic datasets in the absence of native node attributes).}
\label{tab:dataset-stats}
\centering
\small
\begin{tabular}{lrrrrrr}
\toprule
Dataset & $|V|$ & $|E|$ & avg.\,$|e|$ & \#cls & feat.\,dim & CE-hom. \\
\midrule
Citeseer & $3{,}312$  & $1{,}079$  & $3.20$  & $6$  & $3703$ & $0.893$ \\
Pubmed   & $19{,}717$ & $7{,}963$  & $4.35$  & $3$  & $500$  & $0.952$ \\
Cora-CA  & $2{,}708$  & $1{,}072$  & $4.28$  & $7$  & $1433$ & $0.803$ \\
DBLP-CA  & $41{,}302$ & $22{,}363$ & $4.45$  & $6$  & $1425$ & $0.869$ \\
\midrule
Congress & $1{,}718$  & $83{,}105$ & $8.66$  & $2$  & $100$  & $0.555$ \\
Senate   & $282$      & $315$      & $17.17$ & $2$  & $100$  & $0.498$ \\
Walmart  & $88{,}860$ & $69{,}906$ & $6.59$  & $11$ & $100$  & $0.530$ \\
House    & $1{,}290$  & $340$      & $34.73$ & $2$  & $100$  & $0.509$ \\
\bottomrule
\end{tabular}
\end{table}

\paragraph{Citeseer (cocitation).}
Hyperedges represent sets of papers co-cited by another paper. Adapted to the hypergraph format in HyperGCN~\citep{yadati2019hypergcn}. Node features are bag-of-words abstracts (dim $3703$); labels are six paper subject categories.

\paragraph{Pubmed (cocitation).}
Same cocitation construction as Citeseer, sourced via HyperGCN. Features: dim $500$ TF-IDF. Labels: three diabetes subtypes.

\paragraph{Cora-CA (co-authorship).}
Hyperedges represent sets of papers sharing one author, processed by HyperGCN~\citep{yadati2019hypergcn}. Features: dim $1433$ bag-of-words; labels: seven paper topics.

\paragraph{DBLP-CA (co-authorship).}
Co-authorship hyperedges from the DBLP citation network, released via HyperGCN~\citep{yadati2019hypergcn}; the largest co-authorship benchmark in the suite. Features: dim $1425$ bag-of-words; labels: six venue/topic classes.

\paragraph{Congress.}
Bill-cosponsorship network of US Congress members, introduced into the hypergraph-NN literature by ED-HNN~\citep{wang2023edhnn}: nodes are Congresspersons, hyperedges are the sets of (sponsor + co-sponsors) of a bill. With no native attributes, features are imputed as a label-dependent CSBM-style Gaussian (dim $100$, $\sigma = 1$; full specification in the Preprocessing paragraph below). Labels: political party ($2$ classes). Notably dense: average node degree $\approx 427$.

\paragraph{Senate.}
Same construction as Congress, restricted to Senate-only bills. Features: synthetic Gaussian (dim $100$, $\sigma = 1$). Labels: party ($2$ classes). Smallest dataset in the suite ($|V| = 282$); results have unusually high seed-to-seed variance.

\paragraph{Walmart.}
Customer trip data introduced to the hypergraph-NN literature by AllSet~\citep{chien2022allset}: nodes are products, hyperedges are sets of products purchased together in a trip. Features: synthetic Gaussian (dim $100$, $\sigma = 1$). Labels: $11$ product categories. Largest benchmark in the suite ($|V| = 88{,}860$).

\paragraph{House.}
US House of Representatives committee membership network packaged for hypergraph-NN by AllSet~\citep{chien2022allset}: nodes are representatives, hyperedges are committee memberships. Features: synthetic Gaussian (dim $100$, $\sigma = 1$). Labels: party ($2$ classes). Notable for very large hyperedges (avg.\ $|e| \approx 34.7$).

\subsection{Drug Interaction Datasets}
\label{sec:dataset-drugs}

We evaluate HGPM on two higher-order drug interaction datasets, HODDI~\citep{wang2025hoddi} and the Japanese Adverse Drug Event Report (JADER) database~\citep{pmda_jader}\footnote{\url{https://www.info.pmda.go.jp/fukusayoudb/CsvDownload}}. Both are processed into the same hyperedge format---each hyperedge is an observed drug combination labeled with a side-effect class---and use a $50/25/25$ random train/validation/test split. Per-dataset statistics after preprocessing are summarized in Table~\ref{tab:drug-stats}.

\begin{table}[h]
\caption{Statistics of the two drug interaction datasets after preprocessing. The lower block shows the per-combination-size sample count.}
\label{tab:drug-stats}
\centering
\small
\begin{tabular}{lrr}
\toprule
Field & HODDI & JADER \\
\midrule
\#drugs (nodes)            & $1{,}821$ & $4{,}230$ \\
\#hyperedges (samples)     & $72{,}986$ & $126{,}334$ \\
\#side-effect classes      & $1{,}710$ & $100$ (filtered) \\
Combo size range           & $2$--$8$ & $3$--$8$ \\
Mean combo size            & $\sim 2.7$ & $\sim 3.9$ \\
Positive / Negative        & $42{,}113$ / $30{,}873$ & $63{,}167$ / $63{,}167$ \\
Train / Val / Test         & $36{,}493$ / $18{,}246$ / $18{,}247$ & $63{,}116$ / $31{,}508$ / $31{,}710$ \\
\midrule
Size $2$       & $49{,}172$ & --- \\
Size $3$       & $12{,}764$ & $68{,}722$ \\
Size $4$       & $5{,}382$ & $30{,}914$ \\
Size $5$       & $3{,}252$ & $13{,}960$ \\
Size $6$       & $1{,}278$ & $6{,}422$ \\
Size $7$       & $689$ & $4{,}040$ \\
Size $8$       & $449$ & $2{,}276$ \\
\bottomrule
\end{tabular}
\end{table}

\paragraph{HODDI.}
HODDI~\citep{wang2025hoddi} is a higher-order drug--side-effect benchmark derived from FAERS spontaneous-reporting case reports. After deduplication, drug-name normalization, and side-effect filtering, each combination of $2$--$8$ drugs is associated with one of $1{,}710$ side-effect classes. We download HODDI from the authors' release at \url{https://github.com/TIML-Group/HODDI} (MIT license) and apply our $50/25/25$ random split protocol; this differs from the quarter-based temporal split used in the original HODDI paper.

\paragraph{JADER.}
JADER (Japanese Adverse Drug Event Report) is the spontaneous-reporting pharmacovigilance database maintained by Japan's Pharmaceuticals and Medical Devices Agency (PMDA), publicly available at \url{https://www.info.pmda.go.jp/fukusayoudb/CsvDownload}. Each case report is split across four relational tables (DEMO, DRUG, REAC, HIST) joined by a case identifier. We construct hyperedges from the DRUG table by grouping drugs reported in the same case (sizes $3$--$8$ retained) and label each hyperedge with one of $100$ retained MedDRA-coded side-effect classes from the REAC table, following a processing pipeline analogous to HODDI's so that the two datasets share a common hyperedge-and-label format (Table~\ref{tab:drug-stats}). Negatives are sampled to match positives at a $1{:}1$ ratio.

\paragraph{Tasks.}
On both datasets we evaluate two predictive tasks: \emph{edge classification} (predict the dominant side-effect class given an observed drug combination), reported with F1 and AUROC; and \emph{link prediction} (predict whether a drug combination is observed), reported with AUROC and AUPRC. Cross-order generalization (predict higher-order interactions from lower-order observed context) is evaluated only on HODDI.

\section{Implementation Details}
\label{sec:impl-details}

\subsection{Hyperparameters}
\label{sec:hyperparams}

\paragraph{Pretraining.}
The pretrain stage shares the encoder hyperparameters $(d, L, H, \mathrm{dropout})$ with the downstream model selected for the same benchmark, so that the finetune step can load the pretrained encoder without reshaping. The remaining pretraining values are fixed across benchmarks: mask ratio $\rho = 0.2$ over non-center tokens, $50$ epochs with patience-$5$ early stopping on a held-out $10\%$ validation split of the unlabeled corpus, batch size $128$, learning rate $5 \cdot 10^{-4}$, weight decay $5 \cdot 10^{-4}$, gradient clip $\|g\|_2 = 1.0$, and AdamW.

\paragraph{Downstream search.}
For each benchmark we run a per-dataset random search over the discrete grid in Table~\ref{tab:search-space}. The candidate with the highest mean validation accuracy is selected. The optimizer is AdamW throughout. The maximum number of training epochs is held fixed per dataset ($200$ for most benchmarks, $150$ for the larger Walmart variants); the actual stopping point is determined by early stopping on validation accuracy with the patience drawn from the table.

\begin{table}[h]
\caption{Hyperparameter search space for HGPM downstream training. Ranges are the union of values explored across the eight benchmarks via per-dataset random search.}
\label{tab:search-space}
\centering
\small
\begin{tabular}{ll}
\toprule
Hyperparameter & Search range \\
\midrule
\multicolumn{2}{l}{\emph{Encoder.}} \\
Hidden dimension $d$        & $\{32, 64, 128, 256, 512, 768\}$ \\
Layers $L$                  & $\{2, 3, 4, 5, 6\}$ \\
Attention heads $H$         & $\{2, 4, 8, 16\}$ (must divide $d$) \\
Dropout                     & $\{0.0, 0.1, 0.2, 0.3, 0.4, 0.5, 0.6, 0.7\}$ \\
\midrule
\multicolumn{2}{l}{\emph{Inclusion DAG.}} \\
Max order $K_{\max}$        & $\{3, 5, 8, 10, 12, 16, 18, 24\}$ \\
Per-order budget $K_o$      & $\{4, 6, 8, 10, 12, 16, 20\}$ \\
Negatives per positive $r$  & $\{1, 2, 3\}$ \\
Stochastic views $V$        & $2$ (fixed) \\
\midrule
\multicolumn{2}{l}{\emph{Optimization.}} \\
Learning rate               & $\{3{\cdot}10^{-5}, 10^{-4}, 3{\cdot}10^{-4}, 5{\cdot}10^{-4}, 10^{-3}, 3{\cdot}10^{-3}, 10^{-2}\}$ \\
Weight decay                & $\{0, 10^{-5}, 10^{-4}, 5{\cdot}10^{-4}, 10^{-3}, 5{\cdot}10^{-3}, 10^{-2}, 5{\cdot}10^{-2}\}$ \\
Batch size                  & $\{16, 32, 64, 128, 256, 512, 1024\}$ \\
Gradient clip $\|g\|_2$     & $\{0.25, 0.5, 1.0, 2.0, 5.0\}$ \\
Early-stop patience         & $\{3, 5, 8, 10, 12, 15, 20, 50, 100\}$ \\
\bottomrule
\end{tabular}
\end{table}

\subsection{Compute Resources}
\label{sec:compute}

All experiments---pretraining, finetuning, and ablations---are conducted on NVIDIA A40 GPUs (48\,GB). Each training run uses a single GPU. The host machine uses dual Intel Xeon Gold 6326 CPUs (2.90\,GHz, 16 cores each) running Red Hat Enterprise Linux 9.7 with Linux kernel 5.14.

\subsection{Software and Library Versions}
\label{sec:software}

We use Python 3.11 and PyTorch 2.6.0 with the CUDA 12.4 build (cuDNN 9.1). Other key dependencies: NumPy 2.4, SciPy 1.17, scikit-learn 1.8, PyYAML 6.0, and Weights \& Biases 0.26. Exact pinned versions are listed in \texttt{requirements.txt} in the released code.
\section{Comprehensive Ablation Study}
\label{sec:full-ablation}

The six ablations summarized in \S\ref{sec:model-space} (Table~\ref{tab:ablation-model-space}) report only the homophilic and heterophilic group averages for compactness. Table~\ref{tab:ablation-full} below gives the full per-dataset breakdown over all eight benchmarks (Citeseer, Pubmed, Cora-CA, DBLP-CA, Congress, Senate, Walmart, House) for all six ablations. Within each ablation block, the gray cells (also \textbf{bold}) mark the HGPM default, which is best in every column. Numbers are mean test accuracy (\%, $\pm$ std) over 10 seeds.

\begin{table}[t]
\caption{Full per-dataset results for the six HGPM model-space ablations. Within each ablation block, the gray and \textbf{bold} row is the HGPM default; the default attains the best result in every column. Numbers are mean test accuracy (\%, $\pm$ std) over 10 seeds.}
\label{tab:ablation-full}
\centering
\footnotesize
\setlength{\tabcolsep}{3pt}
\resizebox{\textwidth}{!}{%
\begin{tabular}{l cccc cccc}
\toprule
& \multicolumn{4}{c}{\textbf{Homophilic}} & \multicolumn{4}{c}{\textbf{Heterophilic}} \\
\cmidrule(lr){2-5}\cmidrule(lr){6-9}
Setting & Citeseer & Pubmed & Cora-CA & DBLP-CA & Congress & Senate & Walmart & House \\
\midrule
\multicolumn{9}{l}{\textit{(a) Subset Tokens}} \\
\midrule
Observed hyperedges only       & $73.27_{\pm 1.6}$ & $88.04_{\pm 0.7}$ & $82.43_{\pm 1.7}$ & $89.97_{\pm 0.4}$ & $89.73_{\pm 1.3}$ & $71.55_{\pm 5.2}$ & $65.41_{\pm 0.8}$ & $75.13_{\pm 2.3}$ \\
Uncentered random subsets      & $75.42_{\pm 1.4}$ & $88.86_{\pm 0.6}$ & $83.78_{\pm 1.4}$ & $91.07_{\pm 0.3}$ & $91.46_{\pm 1.1}$ & $74.58_{\pm 4.6}$ & $67.43_{\pm 0.6}$ & $76.58_{\pm 1.9}$ \\
No perturbation negatives      & $76.78_{\pm 1.2}$ & $89.51_{\pm 0.5}$ & $84.93_{\pm 0.6}$ & $92.04_{\pm 0.2}$ & $93.42_{\pm 0.9}$ & $76.34_{\pm 5.0}$ & $70.07_{\pm 0.4}$ & $78.61_{\pm 1.6}$ \\
Inclusion DAG tokens           & \cellcolor{gray!20}$\boldsymbol{77.10_{\pm 1.2}}$ & \cellcolor{gray!20}$\boldsymbol{89.70_{\pm 0.5}}$ & \cellcolor{gray!20}$\boldsymbol{85.10_{\pm 0.2}}$ & \cellcolor{gray!20}$\boldsymbol{92.20_{\pm 0.2}}$ & \cellcolor{gray!20}$\boldsymbol{94.10_{\pm 1.0}}$ & \cellcolor{gray!20}$\boldsymbol{77.20_{\pm 4.8}}$ & \cellcolor{gray!20}$\boldsymbol{71.20_{\pm 0.5}}$ & \cellcolor{gray!20}$\boldsymbol{79.20_{\pm 1.8}}$ \\
\midrule
\multicolumn{9}{l}{\textit{(b) Inclusion DAG Edges}} \\
\midrule
No edges (flat tokens)         & $74.13_{\pm 1.8}$ & $87.83_{\pm 0.8}$ & $82.61_{\pm 1.8}$ & $90.07_{\pm 0.5}$ & $89.81_{\pm 1.4}$ & $74.59_{\pm 5.0}$ & $66.41_{\pm 0.9}$ & $75.83_{\pm 2.4}$ \\
Random edges                   & $75.34_{\pm 1.5}$ & $88.46_{\pm 0.7}$ & $83.71_{\pm 1.4}$ & $90.83_{\pm 0.4}$ & $91.69_{\pm 1.2}$ & $75.68_{\pm 4.9}$ & $68.18_{\pm 0.7}$ & $77.04_{\pm 2.1}$ \\
Center-only edges              & $76.21_{\pm 1.3}$ & $89.07_{\pm 0.5}$ & $84.36_{\pm 1.0}$ & $91.62_{\pm 0.3}$ & $92.72_{\pm 1.0}$ & $76.51_{\pm 4.8}$ & $69.62_{\pm 0.6}$ & $78.07_{\pm 1.7}$ \\
Adjacent-order inclusion edges & \cellcolor{gray!20}$\boldsymbol{77.10_{\pm 1.2}}$ & \cellcolor{gray!20}$\boldsymbol{89.70_{\pm 0.5}}$ & \cellcolor{gray!20}$\boldsymbol{85.10_{\pm 0.2}}$ & \cellcolor{gray!20}$\boldsymbol{92.20_{\pm 0.2}}$ & \cellcolor{gray!20}$\boldsymbol{94.10_{\pm 1.0}}$ & \cellcolor{gray!20}$\boldsymbol{77.20_{\pm 4.8}}$ & \cellcolor{gray!20}$\boldsymbol{71.20_{\pm 0.5}}$ & \cellcolor{gray!20}$\boldsymbol{79.20_{\pm 1.8}}$ \\
\midrule
\multicolumn{9}{l}{\textit{(c) Subset-Token Encoder}} \\
\midrule
Mean pooling                       & $73.92_{\pm 1.6}$ & $87.62_{\pm 0.8}$ & $81.83_{\pm 1.4}$ & $89.78_{\pm 0.5}$ & $89.43_{\pm 1.3}$ & $74.37_{\pm 5.4}$ & $65.93_{\pm 0.8}$ & $75.42_{\pm 2.2}$ \\
GRU over linearized sequence       & $71.43_{\pm 3.3}$ & $84.31_{\pm 2.7}$ & $77.84_{\pm 3.9}$ & $84.62_{\pm 3.2}$ & $86.71_{\pm 2.8}$ & $67.89_{\pm 6.3}$ & $64.13_{\pm 2.5}$ & $72.93_{\pm 3.4}$ \\
Transformer w/o structural biases  & $75.61_{\pm 1.5}$ & $88.71_{\pm 0.6}$ & $83.74_{\pm 1.4}$ & $91.06_{\pm 0.4}$ & $91.32_{\pm 1.2}$ & $75.68_{\pm 4.9}$ & $67.83_{\pm 0.7}$ & $77.13_{\pm 2.0}$ \\
Inclusion-aware Transformer        & \cellcolor{gray!20}$\boldsymbol{77.10_{\pm 1.2}}$ & \cellcolor{gray!20}$\boldsymbol{89.70_{\pm 0.5}}$ & \cellcolor{gray!20}$\boldsymbol{85.10_{\pm 0.2}}$ & \cellcolor{gray!20}$\boldsymbol{92.20_{\pm 0.2}}$ & \cellcolor{gray!20}$\boldsymbol{94.10_{\pm 1.0}}$ & \cellcolor{gray!20}$\boldsymbol{77.20_{\pm 4.8}}$ & \cellcolor{gray!20}$\boldsymbol{71.20_{\pm 0.5}}$ & \cellcolor{gray!20}$\boldsymbol{79.20_{\pm 1.8}}$ \\
\midrule
\multicolumn{9}{l}{\textit{(d) Structural Attention Biases}} \\
\midrule
No bias (vanilla attention) & $75.61_{\pm 1.5}$ & $88.71_{\pm 0.6}$ & $83.74_{\pm 1.4}$ & $91.06_{\pm 0.4}$ & $91.32_{\pm 1.2}$ & $75.68_{\pm 4.9}$ & $67.83_{\pm 0.7}$ & $77.13_{\pm 2.0}$ \\
Composition only            & $76.72_{\pm 1.3}$ & $89.46_{\pm 0.5}$ & $84.78_{\pm 0.8}$ & $91.89_{\pm 0.3}$ & $93.47_{\pm 0.9}$ & $76.81_{\pm 5.1}$ & $70.53_{\pm 0.6}$ & $78.58_{\pm 1.7}$ \\
Direction $+$ auxiliary     & $76.13_{\pm 1.4}$ & $89.18_{\pm 0.6}$ & $84.31_{\pm 1.1}$ & $91.56_{\pm 0.3}$ & $92.68_{\pm 1.1}$ & $76.35_{\pm 5.2}$ & $69.43_{\pm 0.6}$ & $77.96_{\pm 1.9}$ \\
Full bias                   & \cellcolor{gray!20}$\boldsymbol{77.10_{\pm 1.2}}$ & \cellcolor{gray!20}$\boldsymbol{89.70_{\pm 0.5}}$ & \cellcolor{gray!20}$\boldsymbol{85.10_{\pm 0.2}}$ & \cellcolor{gray!20}$\boldsymbol{92.20_{\pm 0.2}}$ & \cellcolor{gray!20}$\boldsymbol{94.10_{\pm 1.0}}$ & \cellcolor{gray!20}$\boldsymbol{77.20_{\pm 4.8}}$ & \cellcolor{gray!20}$\boldsymbol{71.20_{\pm 0.5}}$ & \cellcolor{gray!20}$\boldsymbol{79.20_{\pm 1.8}}$ \\
\midrule
\multicolumn{9}{l}{\textit{(e) Pretraining Reconstruction Target}} \\
\midrule
Existence label only    & $73.35_{\pm 2.5}$ & $87.28_{\pm 1.3}$ & $81.88_{\pm 2.4}$ & $90.83_{\pm 1.1}$ & $92.23_{\pm 1.8}$ & $74.56_{\pm 5.1}$ & $68.44_{\pm 1.1}$ & $77.38_{\pm 2.6}$ \\
Raw subset feature      & $75.15_{\pm 1.8}$ & $88.38_{\pm 0.9}$ & $83.14_{\pm 1.8}$ & $91.38_{\pm 0.7}$ & $93.00_{\pm 1.3}$ & $75.48_{\pm 4.9}$ & $69.82_{\pm 0.8}$ & $78.22_{\pm 2.2}$ \\
Composition label only  & $75.90_{\pm 1.4}$ & $89.04_{\pm 0.6}$ & $84.12_{\pm 1.3}$ & $91.65_{\pm 0.5}$ & $93.59_{\pm 1.1}$ & $76.33_{\pm 4.8}$ & $70.40_{\pm 0.6}$ & $78.78_{\pm 1.8}$ \\
TeacherMLP target       & \cellcolor{gray!20}$\boldsymbol{77.10_{\pm 1.2}}$ & \cellcolor{gray!20}$\boldsymbol{89.70_{\pm 0.5}}$ & \cellcolor{gray!20}$\boldsymbol{85.10_{\pm 0.2}}$ & \cellcolor{gray!20}$\boldsymbol{92.20_{\pm 0.2}}$ & \cellcolor{gray!20}$\boldsymbol{94.10_{\pm 1.0}}$ & \cellcolor{gray!20}$\boldsymbol{77.20_{\pm 4.8}}$ & \cellcolor{gray!20}$\boldsymbol{71.20_{\pm 0.5}}$ & \cellcolor{gray!20}$\boldsymbol{79.20_{\pm 1.8}}$ \\
\midrule
\multicolumn{9}{l}{\textit{(f) Token Sequence Order}} \\
\midrule
Ascending order, center first    & $76.11_{\pm 1.8}$ & $89.18_{\pm 0.8}$ & $84.31_{\pm 1.4}$ & $91.63_{\pm 0.6}$ & $92.92_{\pm 1.4}$ & $75.46_{\pm 5.4}$ & $69.73_{\pm 1.0}$ & $78.21_{\pm 2.4}$ \\
Shuffled within order            & $76.57_{\pm 1.5}$ & $89.43_{\pm 0.7}$ & $84.67_{\pm 1.1}$ & $91.88_{\pm 0.4}$ & $93.47_{\pm 1.2}$ & $76.12_{\pm 5.1}$ & $70.41_{\pm 0.8}$ & $78.68_{\pm 2.0}$ \\
Fully shuffled (non-topological) & $76.93_{\pm 1.3}$ & $89.61_{\pm 0.6}$ & $84.93_{\pm 0.8}$ & $92.08_{\pm 0.3}$ & $93.89_{\pm 1.1}$ & $76.81_{\pm 5.0}$ & $70.94_{\pm 0.6}$ & $79.01_{\pm 1.8}$ \\
Descending-order linearization   & \cellcolor{gray!20}$\boldsymbol{77.10_{\pm 1.2}}$ & \cellcolor{gray!20}$\boldsymbol{89.70_{\pm 0.5}}$ & \cellcolor{gray!20}$\boldsymbol{85.10_{\pm 0.2}}$ & \cellcolor{gray!20}$\boldsymbol{92.20_{\pm 0.2}}$ & \cellcolor{gray!20}$\boldsymbol{94.10_{\pm 1.0}}$ & \cellcolor{gray!20}$\boldsymbol{77.20_{\pm 4.8}}$ & \cellcolor{gray!20}$\boldsymbol{71.20_{\pm 0.5}}$ & \cellcolor{gray!20}$\boldsymbol{79.20_{\pm 1.8}}$ \\
\bottomrule
\end{tabular}%
}
\end{table}

\paragraph{(a) Subset Tokens.}
The four rows isolate three independent design choices via three pairwise comparisons. Substructure expansion (Observed hyperedges only vs.\ No perturbation negatives, both centered) contributes the largest average gain among the three comparisons, with the gap concentrated on heterophilic and sparse-hyperedge data and shrinking on saturated co-citation data (see Table~\ref{tab:ablation-full}). The mechanism is that subset tokens expose the COMP/EMER/INHIB structure of Eq.~\eqref{eq:comp-label}, which a representation built only from hyperedges considered in isolation does not carry. Center anchoring (Uncentered random subsets vs.\ No perturbation negatives) contributes a smaller but consistent gain across datasets by concentrating the fixed token budget on subsets containing the target node $c$. Sampled negatives (No perturbation negatives vs.\ Default) yield the smallest contribution on average and are dataset-dependent, falling within noise on saturated benchmarks while providing a modest gain on heterophilic ones where INHIB regimes carry real signal. Substructure expansion thus accounts for most of the gain, with negative-token regularization a secondary effect.

\paragraph{(b) Inclusion DAG Edges.}
The four rows isolate three properties of the structural attention bias. The total contribution of the structural bias (Default vs.\ No-edges) is $+3.07$pp on average and scales with hyperedge density, with the largest gaps on dense heterophilic data and the smallest on saturated co-citation data. Random edges recover a minority of this gain while center-only edges recover the majority, indicating that the \emph{content} of inclusion edges, namely adjacent-order subset-to-superset adjacency, matters far more than their quantity. Center anchoring is therefore the dominant structural prior, with adjacent-order topology providing the remaining signal that uniquely encodes the COMP/EMER/INHIB regimes (since composition labels are defined only on adjacent-order pairs).
\textbf{Cross-table tie-in.} Tables (a) and (b) independently identify center anchoring as the dominant factor: (a) at the tokenization layer (subsets must contain $c$) and (b) at the attention-topology layer (every token attends preferentially to $c$).

\paragraph{(c) Subset-Token Encoder.}
The ranking GRU $<$ mean pooling $<$ Transformer without structural biases $<$ inclusion-aware Transformer (default) holds across datasets, with each step contributing a different mechanism. The GRU consistently underperforms mean pooling, which is consistent with the inclusion-DAG token sequence being permutation-invariant so that imposing a sequential order is uninformative; the effect is largest on small-data Senate, where the absolute GRU gap to default is the widest. Mean pooling is a strong baseline that recovers roughly half of the Transformer's improvement via feature averaging alone, and the vanilla Transformer adds attention contextualization on top. The structural bias contributes the remaining gap (Default vs.\ Vanilla), and is largest on heterophilic dense data and smallest on saturated co-citation data.
\textbf{Cross-table tie-in.} The (c) Transformer-w/o-biases row is identical to the (d) No-bias row. The (b) No-edges (flat) row sits below (c) Transformer-w/o-biases because (b) additionally removes the topological position embedding; this residual gap is several times larger than the (f) shuffled-vs-default gap, consistent with position embedding becoming the primary carrier of topological signal once the structural bias is disabled.

\paragraph{(d) Structural Attention Biases.}
This decomposes the structural attention bias of Eq.~\eqref{eq:attn-bias} into its semantic ($b^{\mathrm{comp}}$) and topological ($b^{\mathrm{dir}} + b^{\mathrm{aux}}$) components. The total bias gap (Default vs.\ No-bias) averages about $1.84$pp and scales with hyperedge density, mirroring the structural bias gap in (b). Composition labels alone recover roughly three quarters of this gap on average, with the recovery stable across benchmarks regardless of homophily, saturation, or scale. Direction $+$ auxiliary biases together recover a smaller fraction on average; the two recoveries are not additive, which is consistent with partial redundancy between composition and topology biases. This supports the design choice in Eq.~\eqref{eq:attn-bias} of including $b^{\mathrm{comp}}$ as an explicit attention bias rather than relying on topology priors alone.

\paragraph{(e) Pretraining Reconstruction Target.}
Unlike the previous rows, which ablate architecture, (e) ablates the pretraining target while keeping the architecture fixed. The four targets form a clear hierarchy on every benchmark: existence-only (predicting $\mathds{1}_S$) underperforms even training from scratch, raw subset features give modest gains, composition labels (discrete COMP/EMER/INHIB supervision) recover the bulk of the gap, and the TeacherMLP target attains the default. Existence-only supplies too narrow a signal: the encoder fits a binary indicator that does not transfer to downstream node classification. Raw subset features are trivially recoverable from mean-pooling, which limits their pretraining utility. Composition labels carry most of the semantic signal targeted by $\mathcal{L}_{\mathrm{sem}}$, and the continuous TeacherMLP alignment closes the remaining gap that discrete supervision leaves open.

\paragraph{(f) Token Sequence Order.}
This row probes sensitivity to the token order fed to the Transformer. All linearization variants stay close to default, with sub-1pp gaps from the fully shuffled setting on average, indicating that structural attention bias rather than position embedding carries the topological signal. The descending-order linearization remains the strongest by a small margin, and the model is largely order-robust as long as the bias terms remain active.
\textbf{Cross-table tie-in.} The shuffled-vs-default gap in (f) is much smaller than the No-edges to Transformer-without-biases gap spanning (b) and (c). Both comparisons contrast topology-via-position with no-position, but in (f) the structural bias is \emph{active} and carries the topological signal redundantly, whereas in the (b) to (c) comparison the bias is \emph{off} and position embedding is the only carrier, producing a substantially larger gap.

\paragraph{Cross-axis trends: heterophilic data benefits more from architectural priors.}
A consistent pattern emerges across the four architectural ablations from (a) through (d): each non-catastrophic design choice has a larger impact on the heterophilic group than on the homophilic group, holding for Hyperedge-only in (a), No-edges in (b), Vanilla Transformer in (c), and No-bias in (d). The catastrophic GRU row of (c) is the exception, since both groups collapse to a similar absolute degree when the inductive bias is fundamentally wrong rather than merely missing. We note that heterophilic datasets also have lower base accuracy and therefore more headroom, which contributes to the larger absolute gaps. The pretraining-target ablation (e) does \emph{not} show the same asymmetry; existence-only even widens the gap slightly more on the homophilic group, suggesting that representation-quality choices and inductive-bias choices have qualitatively different per-group impact, and that HGPM's heterophilic advantage is more closely tied to architectural priors than to the pretraining target. Linearization (f) follows the architectural pattern at much smaller magnitudes. This per-group asymmetry across architectural axes is consistent with HGPM's largest lead in Table~\ref{tab:node-classification} appearing on Senate, Walmart, and House, where message passing on observed hyperedges has limited reach.

\section{Case Study: HGPM vs.\ Baseline}
\label{sec:case-supp}

We complement the case study in Figure~\ref{fig:case} with two head-to-head comparisons against the strongest feature-similarity baseline \citep{xie2025khgnn} on the same three candidates added to FOLFOX (panitumumab, bevacizumab, capecitabine). Both panels share the JADER-derived ground truth used in the main text: panitumumab preserves peripheral neuropathy (positive, COMP), while bevacizumab and capecitabine suppress it (negative, INHIB).

\paragraph{Link-prediction score.} Figure~\ref{fig:case-link-score} contrasts each model's binary score against the truth label. HGPM separates all three correctly: $0.84$ for panitumumab (positive), $0.29$ for bevacizumab and $0.22$ for capecitabine (both negative, well below the $0.5$ threshold). The baseline scores all three above the threshold ($0.58 / 0.63 / 0.54$) and therefore flags bevacizumab and capecitabine as positive, in opposition to the data. The mechanism is direct: the baseline reads the candidate from a drug-feature encoder, and the three additions sit close in feature space, so the score tracks drug similarity rather than the joint observation pattern of the $3$-drug subset and its $4$-drug superset.

\begin{figure}[h]
\centering
\includegraphics[width=0.95\textwidth]{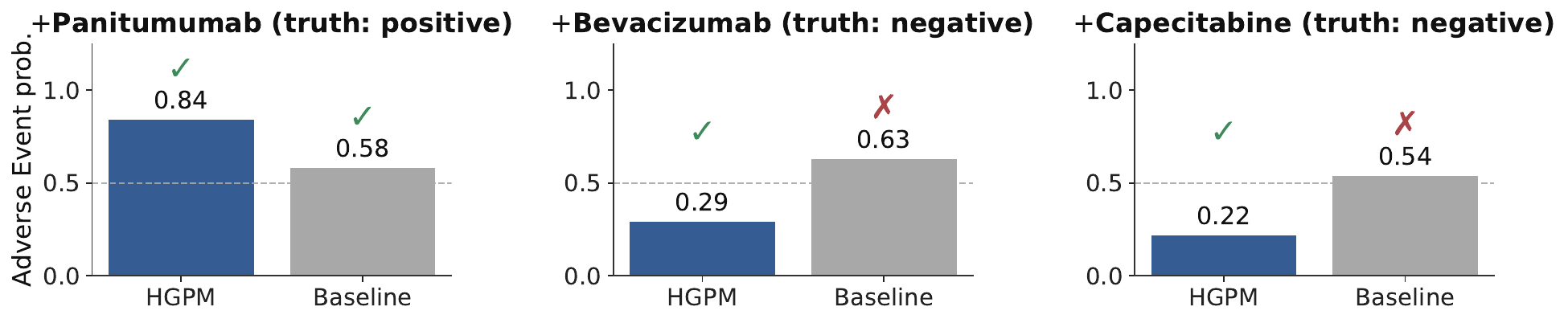}
\caption{Binary link-prediction score for each candidate added to FOLFOX, comparing HGPM with the feature-similarity baseline. Truth label is shown above each card; the dashed line marks the $0.5$ decision threshold. HGPM separates all three branches; the baseline misclassifies bevacizumab and capecitabine as positive.}
\label{fig:case-link-score}
\end{figure}

\paragraph{Adverse-event distribution.} Figure~\ref{fig:case-baseline-ae} shows the baseline's top-$3$ predicted adverse events per candidate. Peripheral neuropathy is ranked first across all three branches, with probabilities $0.44 / 0.26 / 0.28$, mirroring the link-prediction failure: the baseline averages all three candidates toward a shared neuropathy prior. HGPM (Figure~\ref{fig:case}d) demotes peripheral neuropathy to rank $2$ for bevacizumab and rank $3$ for capecitabine, surfacing regimen-specific signatures instead: hypertension and proteinuria for the anti-VEGF branch, and hand-foot syndrome and diarrhea for the fluoropyrimidine branch. These are exactly the dose-limiting toxicities expected on clinical grounds, recovered without any clinical labels beyond the binary positive/negative regimen indicators. The contrast highlights the clinical value of compositional reasoning: a similarity model collapses distinct regimens onto a shared adverse-event profile, whereas HGPM recovers the distinct profile each regimen actually produces.

\begin{figure}[h]
\centering
\includegraphics[width=0.95\textwidth]{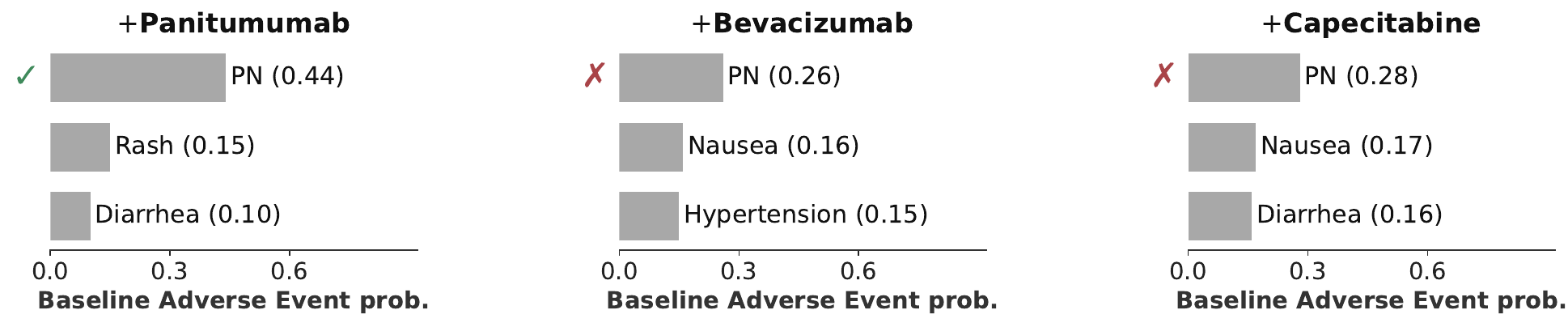}
\caption{Top-$3$ adverse events predicted by the feature-similarity baseline for each candidate added to FOLFOX. Peripheral neuropathy ranks first across all three branches, indicating that the baseline cannot resolve regimen-specific adverse-event signatures from drug-feature similarity alone.}
\label{fig:case-baseline-ae}
\end{figure}


\end{document}